\documentclass[11pt]{article}

\usepackage[preprint]{acl}

\usepackage{times}
\usepackage{latexsym}
\usepackage{amsmath}

\usepackage{graphicx}
\usepackage{caption}
\usepackage{subcaption} 

\usepackage{graphicx}
\usepackage{xcolor} 

\usepackage{subcaption} 
\usepackage{tikz}
\usetikzlibrary{calc,arrows.meta,positioning}

\usepackage{makecell}

\usepackage[T1]{fontenc}

\usepackage[utf8]{inputenc}

\usepackage{microtype}

\usepackage{inconsolata}

\usepackage{graphicx}
\usepackage[table]{xcolor}

\usepackage[capitalize,noabbrev,nameinlink]{cleveref}
\usepackage{todonotes}
\usepackage{amssymb,amsthm}

\usepackage{booktabs}
\usepackage{graphicx}
\usepackage{multirow}
\usepackage{xcolor}
\usepackage{colortbl}
\usepackage{multirow}

\usepackage{enumitem}
\usepackage{tcolorbox}

\newtheorem{globalcounter}{Theorem}[section] 

\newtheorem{theorem}[globalcounter]{Theorem}

\newtheorem{lemma}[globalcounter]{Lemma}

\newtheorem{remark}[globalcounter]{Remark}

\renewcommand{\paragraph}[1]{\noindent\textbf{#1}}

\title{Hierarchical Token Prepending: Enhancing Information Flow in Decoder-based LLM Embeddings}

\author{
  Xueying Ding\textsuperscript{1}\thanks{Equal contribution.}\thanks{Work done as intern in Snap Inc.},
  Xingyue Huang\textsuperscript{2}\footnotemark[1]\footnotemark[2],
  Mingxuan Ju\textsuperscript{3},
  Liam Collins\textsuperscript{3},\\
  \textbf{Yozen Liu}\textsuperscript{\textbf{3}},
  \textbf{Leman Akoglu}\textsuperscript{\textbf{1}},
  \textbf{Neil Shah}\textsuperscript{\textbf{3}},
  \textbf{Tong Zhao}\textsuperscript{\textbf{3}}
  \\[0.6ex]
\textsuperscript{1}Carnegie Mellon University\;
\textsuperscript{2}University of Oxford\;
\textsuperscript{3}Snap Inc.
}

\begin{document}
\maketitle
\begin{abstract}
Large language models produce powerful text embeddings, but their causal attention mechanism restricts the flow of information from later to earlier tokens, degrading representation quality. While recent methods attempt to solve this by prepending a single summary token, they over-compress information, hence harming performance on long documents. We propose \emph{Hierarchical Token Prepending} (HTP)
\footnote{\url{https://github.com/snap-research/HTP}}, 
a method that resolves two critical bottlenecks. To mitigate attention-level compression, HTP partitions the input into blocks and prepends block-level summary tokens to subsequent blocks, creating multiple pathways for backward information flow. To address readout-level over-squashing, we replace last-token pooling with mean-pooling, a choice supported by theoretical analysis. HTP achieves consistent performance gains across 11 retrieval datasets and 30 general embedding benchmarks, especially in long-context settings. As a simple, architecture-agnostic method, HTP enhances both zero-shot and finetuned models, offering a scalable route to superior long-document embeddings.
\end{abstract}

\section{Introduction}

Neural text embeddings play a critical role in applications such as information retrieval, recommendation, document clustering, etc. Notably, large language models (LLMs) have emerged as a powerful paradigm for embedding models, offering remarkable zero-shot capabilities and often surpassing traditional sentence embedding models~\citep{muennighoff2023mteb,thakur2021beir}.

However, LLMs are optimized for auto-regressive token generation, not for producing sequence-level embeddings. Therefore, LLMs typically rely on causal attention instead of the bidirectional encoding used in traditional encoder-only models~\citep{devlin2019bert,liu2019roberta}. As a result, embedding generation suffers from \emph{restricted backward flow}~\citep{echoembedding,fu2024token}, as earlier tokens cannot access later positions in the sequence.
The representations of earlier tokens cannot integrate later information, and hence degrade downstream performances.

Multiple recent works~\citep{echoembedding,fu2024token,xu-etal-2024-reading} have attempted to introduce such backward flow with minimal or no changes to model architectures, thereby preserving zero-shot capabilities. Echo embedding~\citep{echoembedding} addresses the issue at the \emph{token level} by repeating the input sequence, 
allowing early tokens in the second copy can attend to later tokens in the first. While effective, doubling the sequence 
is computationally infeasible for long documents. In contrast, token prepending (TP)~\citep{fu2024token} is
a \emph{document level} solution that dynamically rewire the end-of-sequence (EOS) representation to the beginning of the input with a custom \texttt{<PST>} token before the next layer. This allows all tokens to attend to a compressed summary of the full sequence,
yet degrades performance on long-context tasks due to over-compression. We therefore ask: \emph{Can we preserve backward flow while alleviating information compression and maintaining scalability for long-context tasks?} 

In this work, we investigate the information over-compression effect in long-context LLM embeddings generations with TP and identify two bottlenecks: (1) an \textbf{attention-level} bottleneck, where a single prepended token is forced to summarize the entire document, and (2) a \textbf{readout-level} bottleneck, where the final embedding is taken solely from the last token, resulting in an over-squashed representation~\citep{barbero2024transformers}.

We propose \emph{Hierarchical Token Prepending} (HTP), which mitigates over-compressing by replacing a single global \texttt{<PST>} token with a hierarchy of block-level summary tokens. HTP partitions the input into blocks and assigns each block a designated summary token prepended to subsequent blocks, enabling backward flow through multiple pathways and alleviating the attention-level bottleneck. For the readout-level bottleneck, we replace last-token pooling with mean-pooling and show improved performance in long-context settings.

Our contributions can be summarized as follows:
\begin{itemize}[leftmargin=0.4cm,nosep]
\item We propose \emph{Hierarchical Token Prepending} (HTP), 
a block-level summary tokens prepending method that enables backward flow with less over-compression.
\item We show that \emph{mean-pool} readout is more suitable for long-context retrieval tasks, with empirical and theoretical evidence.
\item HTP showcases consistent improvements over extensive evaluation across \(11\) retrieval datasets and \(30\) general embedding tasks, under standard and long-context settings.
HTP can also improve performance of finetuned embedding models (e.g., \emph{NV-Embed-v2}~\citep{lee2025nvembed}).
\end{itemize}

\section{Related Work}

\paragraph{Sentence Embeddings.}
The foundation of modern text embeddings was built on bidirectional encoders like BERT~\citep{devlin2019bert} and RoBERTa~\citep{liu2019roberta}. While effective for token-level tasks, their raw outputs proved suboptimal for sentence similarity tasks. \
This motivated a wave of research focused on specialized training strategies, including supervised finetuning with SBERT~\citep{reimers-2019-sentence-bert} and Sentence-T5~\citep{ni-etal-2022-sentence}, contrastive learning with SimCSE~\citep{gao2021simcse} and SNCSE~\citep{chanchani-huang-2023-composition}, and prompt-based tuning with PromptBERT~\citep{jiang-etal-2022-promptbert}. Although powerful, these methods require extensive, task-specific training or finetuning, thus limiting their broader applicability.

\paragraph{Finetuning LLMs for Embeddings.}
More recently, a line of work has focused on adapting decoder-only LLMs for embedding tasks through architectural modifications and finetuning. 
LLM2Vec~\citep{behnamghader2024llmvec} and GRIT~\citep{muennighoff2025generative} apply contrastive finetuning to make LLM representations suitable for retrieval. 
Deelm~\citep{li-li-2024-bellm} modify a decoder-only model to be bidirectional, resembling traditional encoders. Similarly, NVEmbed~\citep{lee2025nvembed} and RepLLaMA~\citep{finetunellama} have demonstrated strong performance by specifically training LLMs for retrieval tasks. 
Our proposed method, HTP, is orthogonal to these approaches and can be applied to enhance the performance of these finetuned models, as we show in \Cref{sec:ablation}.

\paragraph{Training-free LLM for Embeddings.}
A key appeal of LLMs is their ability to generate powerful embeddings in a zero-shot, training-free manner. 
Literature proposed strategies such as optimizing prompts to elicit better representations~\citep{prompteol,lei-etal-2024-meta,pretended_cot,cheng2025contrastive} and utilizing expert routers in MoE LLMs~\citep{li2025your}.
However, a core challenge for all training-free methods is the \emph{restricted backward flow} inherent from the causal attention mechanism of autoregressive models.

Two recent methods directly tackle this issue. 
Echo Embedding~\citep{echoembedding} duplicates the input sequence, allowing the second copy to attend to the first, but resulting doubled sequence length. 
In contrast, Token Prepending~\citep{fu2024token} redirects a final summary token to the beginning of the sequence to create a global view. While more efficient, this creates an information bottleneck, leading to over-compression and degraded performance on long documents. 
Our work, HTP, directly addresses this bottleneck by introducing a hierarchy of block-level summary tokens, enabling robust backward information flow without over-compression.

\section{Methodological Preliminaries}

We begin by examining two factors critical to the quality of embeddings from decoder-only LLMs: \textit{(i)} the choice of readout function for aggregating token representations, and \textit{(ii)} the necessity of enabling backward dependencies in the causal attention mechanism.

\subsection{Mean vs.\ Last-token Embedding}
\label{subsec:mean-vs-last}

Two strategies dominate the aggregation of token embeddings are mean-pooling (averaging all token representations) and last-token pooling (using only the final token's representation). While last-token pooling is often favored for its simplicity and alignment with the autoregressive design of LLMs, it creates a significant information bottleneck. Recent work suggests it is sensitive to prompt selection \citep{echoembedding}, and our experiments in \Cref{sec:experiment} show it underperforms in retrieval tasks.

To understand this performance gap, we show that mean-pooling is more robust to the ``over-squashing'' issue by distributing representational importance across all tokens instead of compressing it into a single position. Following sensitivity analysis adapted for decoder-only Transformers \citep{barbero2024transformers, topping2022understanding}, we quantify how the final embedding changes in response to a perturbation in an input token.
Let $\mathbf{v}^{(0)}_i\in\mathbb{R}^d$ be the $i$-th input token embedding, and  $\mathbf{y}_1,\ldots,\mathbf{y}_n\in\mathbb{R}^d$ be the post-normalization representations at the final layer.
We study the Jacobian norms
$
\big\|\tfrac{\partial \mathbf{y}_n}{\partial \mathbf{v}^{(0)}_i}\big\|$ for last-token readout and 
$\big\|\tfrac{\partial \bar{\mathbf{y}}}{\partial \mathbf{v}^{(0)}_i}\big\|$ for mean-token readout\footnote{$\|\cdot\|$ denotes the Euclidean norm on vectors and the induced operator norm on Jacobians.}, where $\bar{\mathbf{y}}:=\tfrac{1}{n}\sum_{j=1}^n\mathbf{y}_j$. 

\begin{theorem}[Mean vs.\ Last-token Embedding]\label{thm:informal-mean-vs-last}
In a causal, decoder-only Transformer with $L$ layers, there exists a depth-dependent constant
$K_L>0$ and a nonnegative, lower-triangular,
row-stochastic mixing matrix $\mathbf{A}\in\mathbb{R}^{n\times n}$ (capturing aggregate
attention+residual flow across layers) such that, for every input position $i\in[n]$,
\vspace{-0.1in}
\begin{align}
\underbrace{\Big\|\tfrac{\partial \mathbf{y}_n}{\partial \mathbf{v}^{(0)}_i}\Big\|}_{\text{last-token}}
&\le\; K_L\,\mathbf{A}_{n,i}, \\
\underbrace{\Big\|\tfrac{\partial \bar{\mathbf{y}}}{\partial \mathbf{v}^{(0)}_i}\Big\|}_{\text{mean-tokens}}
&\le\; \frac{K_L}{n}\sum_{j=1}^n \mathbf{A}_{j,i}.
\end{align}
\vspace{-0.15in}
\end{theorem}

\paragraph{Interpretation.} The last-token bound depends on a \emph{single} entry \(\mathbf{A}_{n,i}\), representing the influence of input token i on the final token $n$, which can diminish rapidly with network depth. 
In contrast, the mean-pooling bound aggregates the \emph{entire column} \(\sum_j \mathbf{A}_{j,i}\) and is therefore topologically depth-agnostic up to the scale factor \(K_L\), implying more robustness against over-squashing.
We present the formal definitions 
and proof in \cref{sec:proofs}.

\paragraph{Long-context Sentence Similarity (STS) Task.} 
We further empirically validate the impact of over-squashing on a modified STS task, where create long-context inputs by concatenating sentences with similar human-annotated scores (details in \cref{sec:experiment_details}). As shown in \Cref{fig:longsts}, the performance of last-token embeddings degrades sharply as input length increases. While methods like PromptEOL~\citep{prompteol} offer some mitigation, they still cannot match the stability and superior performance of mean-pooling.

\begin{figure}[h!]
    \centering
    \includegraphics[width=\linewidth]{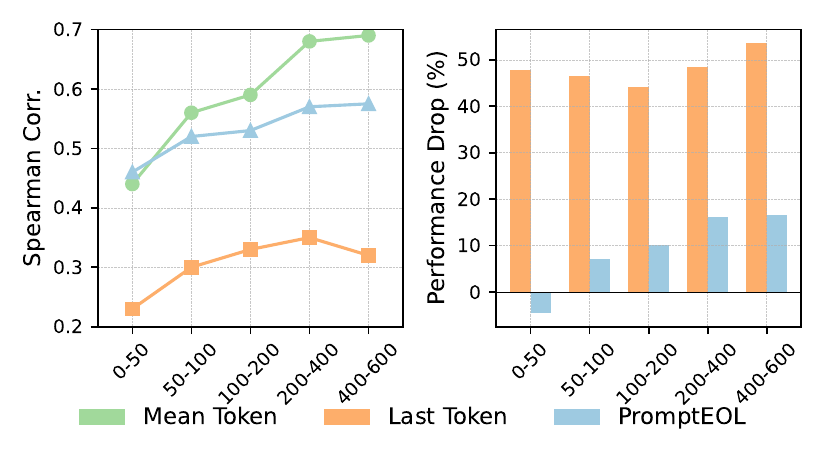} 
    \vspace{-0.3in}
    \caption{Left: Performance by embedding method across sentence lengths. Right: Performance drop of last-token and PromptEOL vs. mean at longer lengths.}
    \label{fig:longsts}
    \vspace{-0.2in}
\end{figure}

\subsection{Backward Dependency}

The causal attention in decoder-only LLMs restricts backward information flow, as tokens cannot attend to subsequent positions. Input repetition methods~\citep{echoembedding,xu-etal-2024-reading} address this by allowing a ``second pass'' over the input, where the repeated sequence can attend to the original to form a more complete representation. \Cref{fig:combined_figure} showcases a masking experiment to isolate the effect of this backward dependency. We find that preventing the second-pass tokens from attending to the first-pass tokens significantly degrades STS performance, whereas masking the forward attention (from the first pass to the second) has a negligible impact. This result confirms that enabling backward information flow is critical for generating high-quality embeddings.

\begin{figure}[h!]
    \centering

    \begin{subfigure}[b]{0.45\linewidth}
        \centering
        \includegraphics[width=\linewidth]{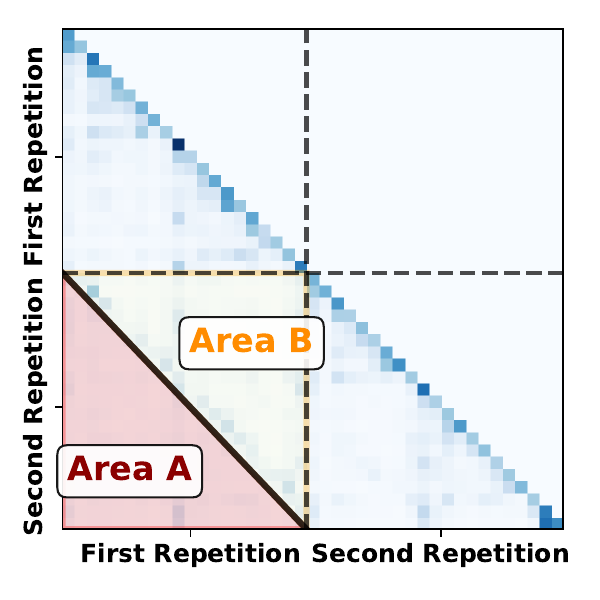}
        \vspace{-0.25in}
        \caption{Attention Map.}
        \label{fig:echo_atten}
    \end{subfigure}
    \hfill
    \begin{subfigure}[b]{0.5\linewidth}
    \scalebox{0.75}{
        \centering
        \vspace{-0.2in}
        \begin{tabular}{l c}
            \toprule
            \textbf{Condition} & \textbf{STS Score} \\
            \midrule
            Echo Unmasked & \textbf{68.00} \\
             + Mask Area A & 67.89 \\
            + Mask Area B & 54.25 \\
            \bottomrule
        \end{tabular}
        }
        \caption{Average scores across 4 STS tasks on Echo Mean Embeddings, with no-masking applied, masking Area A and masking Area B.}
        \label{tab:sts_scores} 
    \end{subfigure}
    \vspace{-0.1in}
    \caption{Disabling second-pass backward attention significantly degrades STS performance.}
    \label{fig:combined_figure}
    \vspace{-0.2in}
\end{figure}

\begin{figure}[h!]
    \centering
    \includegraphics[width=0.9\linewidth]{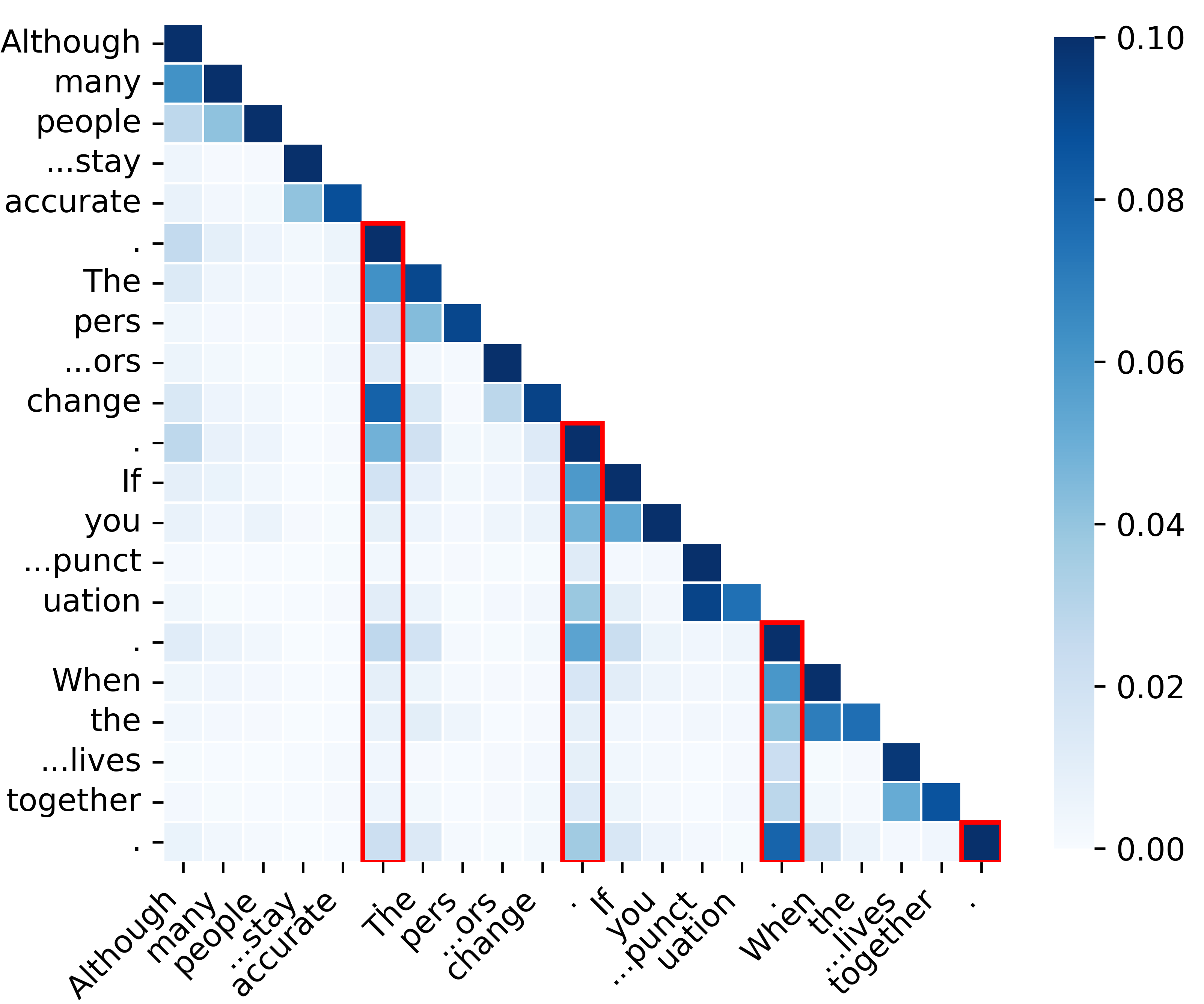} 
    \vspace{-0.1in}
    \caption{The attention map shows that End-of-Sentence (EOS) tokens are capturing information with more attention lookup to previous tokens.}
    \vspace{-0.2in}
\label{fig:sent_atten}
\end{figure}

\begin{figure*}[!ht]
    \centering
    \vspace{-0.3in}
    \includegraphics[width=0.95\linewidth]{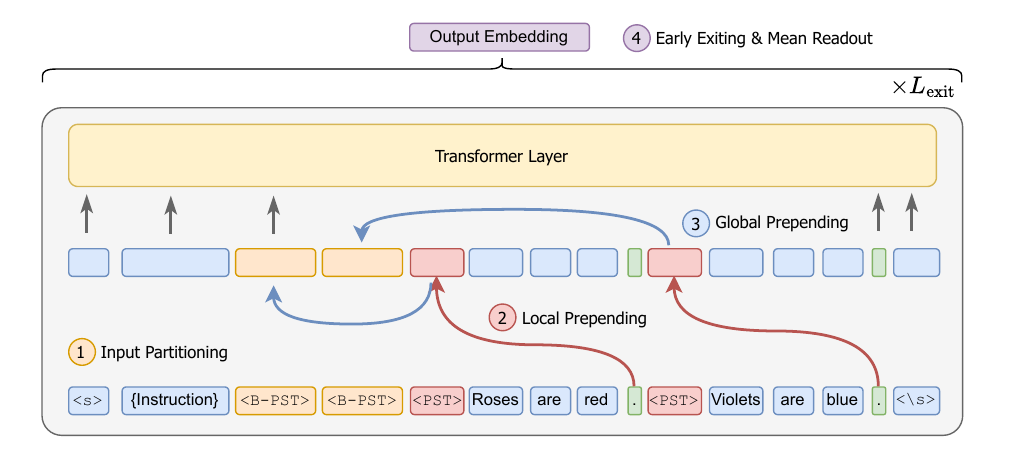} 
    \vspace{-0.2in}
    \caption{
    HTP partitions the input into blocks and creates a two-level summary. First, the hidden state of each block's final token is copied to a local summary token (\texttt{<PST>}). These local summaries are then aggregated into a global (\texttt{<B-PST>}) block at the front, making them accessible to all tokens. A mean readout with early exiting produces the final embedding.}
    \label{fig:sample-image}
    \vspace{-0.1in}
\end{figure*}

\section{Hierarchical Token Prepending}
\label{sec:method}

To mitigate the attention and readout bottlenecks in decoder-only models, we propose \emph{Hierarchical Token Prepending} (HTP), a training-free method that establishes multi-level backward information flow. As illustrated in \Cref{fig:sample-image}, HTP operates in three main stages:

\begin{enumerate}[leftmargin=*,nosep]
    \item \textbf{Input Partitioning:} 
    The input text is segmented into semantic blocks (e.g., sentences), and placeholder summary \texttt{<PST>} tokens are inserted to create a hierarchical structure.
    \item \textbf{Local Prepending:} 
    Between Transformer layers, the final hidden state of each block is ``rewired'' to a corresponding local summary token, creating a block-level summary.
    \item \textbf{Global Prepending:} 
    These local summaries are then propagated to a global summary \texttt{<B-PST>} block at the beginning of the sequence, making them accessible to all tokens.
\end{enumerate}

After the final layer, a mean-pooling readout is used to produce the output embedding.
The following subsections elaborate over each step.

\subsection{Input Partitioning}
\label{sec:input_partition}

HTP begins by restructuring the input sequence to accommodate hierarchical summaries. Unlike methods that use a single summary token for entire text, HTP partitions the text into smaller semantic units to prevent information over-compression.

Given an input sequence of tokens $T = [t_1,\cdots,t_n]$, we first segment it into $M$ subsequences $(S_1,\cdots,S_M)$, $S_m$ with $m \in [M]$ represents a local context. This partition is defined by indices $0 = i_0 < i_1 < \cdots < i_M = n$, s.t. $S_m = (t_{i_{m-1}+1}, t_{i_{m-1}+2}, \dots, t_{i_m})$ for $m \in [M]$.

We then augment the partitioned sequence with two types of special placeholder tokens to create the input sequence $T'$:
\begin{enumerate}[leftmargin=*,nosep]
    \item A local summary token, $\texttt{<PST>}_m$, is inserted before each subsequence $S_m$.
    \item A block of $M$ global summary tokens, $(\texttt{<B-PST>}_1, \dots, \texttt{<B-PST>}_M)$, is prepended to the beginning of the entire sequence.
\end{enumerate}

Hence, the resulted augmented structure of the input sequence $T'$ is:
\vspace{-0.1in}
\begin{align}
T' = [&\texttt{<B-PST>}_1, \dots, \texttt{<B-PST>}_M, \\
&\texttt{<PST>}_1, S_1,  \dots, \texttt{<PST>}_M, S_M] \nonumber
\vspace{-0.1in}
\end{align}

In practice, we define each subsequence $S_m$ by grouping every $K$ sentences, leveraging their natural semantic boundaries. This choice is motivated by observed attention patterns in decoder-only models, where the EOS tokens effectively aggregate information from preceding words (\Cref{fig:sent_atten}). Thus, the final token's hidden state serves as an effective proxy for a sentence's summary. Unless otherwise specified, we set the hyperparameter $K=1$. An example is provided in \Cref{app:examples}.

\subsection{Local Prepending}
\label{sec:local_prepending}

After partitioning, the Local Prepending step populates the \texttt{<PST>} placeholders with sentence-level summaries. This operation is performed dynamically between the layers of the Transformer.

Let $\mathbf{v}^{(\ell)} = (\mathbf{v}_1^{(\ell)}, \dots, \mathbf{v}_{|T'|}^{(\ell)})$ be the sequence of hidden states entering layer $\ell$. Since the placeholder tokens are not in the model's trained vocabulary, their embeddings are randomly initialized for the first layer. For each subsequence layer $\ell > 1$, we apply a ``rewiring'' function, $f_{\text{local}}$, before the self-attention mechanism. This function copies the hidden state of the final token of each subsequence $S_m$ to the position of its corresponding local summary token, $\texttt{<PST>}_m$.

Formally, let $\texttt{pos}(\cdot)$ be a function returning a token's index, and $\texttt{end}(S_m)$ be the final token of subsequence $S_m$, which is typically sentence-ending punctuation (e.g., `\texttt{.}', `\texttt{!}', `\texttt{?}'). Let $\mathcal{P}_{\text{PST}} = \{ \text{pos}(\texttt{<PST>}_m) \}_{m=1}^M$ be the set of all local summary token positions, and the mapping $\mu: \mathcal{P}_{\text{PST}} \to [M]$ returns the sentence index $m$ for a given token position $i \in \mathcal{P}_{\text{PST}}$. The rewiring function $f_\text{local}$ for a hidden state $\mathbf{v}_i^{(\ell)}$ is defined as:
\vspace{-0.1in}
\begin{equation*}
(\,f_{\text{local}}(\mathbf{v}^{(\ell)}))_i =
\begin{cases}
    \mathbf{v}^{(\ell)}_{\text{pos}(\text{end}(S_{\mu(i)}))} & \text{if } i \in \mathcal{P}_{\text{PST}}, \\
    \mathbf{v}_i^{(\ell)} & \text{otherwise.}
\end{cases}
\vspace{-0.1in}
\end{equation*}

This rewired sequence, $\mathbf{v}'^{(\ell)} = f_{\text{local}}(\mathbf{v}^{(\ell)})$, is then fed into the attention block of layer $\ell$.
This process ensures that each $\texttt{<PST>}_m$ token carries a summary of subsequence $S_m$, establishing a \emph{sentence-level} backward dependency.

\subsection{Global Prepending}
\label{sec:global_prepending}

Building upon the local summaries, the Global Prepending step enables \emph{document-level} backward flow. In this stage, we propagate the sentence-level summaries from the \texttt{<PST>} tokens into the global \texttt{<B-PST>} block at the start of the sequence. This hierarchical design creates multiple backward information pathways, mitigating the bottleneck of a single summary token.

We define a second rewiring function, $f_{\text{global}}$, which operates on the output of the local prepending step, $\mathbf{v}'^{(\ell)}$.
$f_{\text{global}}$ copies the hidden state from each local summary token $\texttt{<PST>}_m$ to its corresponding global summary token $\texttt{<B-PST>}_m$.
Let $\mathcal{P}_{\text{BPST}} = \{ \text{pos}(\texttt{<B-PST>}_m) \}_{m=1}^M$ denote the set for the global token positions, and mapping $\nu(i)$ returns the sentence index $m$ for a position $i \in \mathcal{P}_{\text{BPST}}$. 
The $f_{\text{global}}$ function is then defined as:
\vspace{-0.1in}
\begin{equation*}
(\,f_{\text{global}}(\mathbf{v}'^{(\ell)}))_i =
\begin{cases}
    \mathbf{v}'^{(\ell)}_{\text{pos}(\texttt{<PST>}_{\nu(i)})} & \text{if } i \in \mathcal{P}_{\text{BPST}}, \\
    {\mathbf{v}'}_i^{(\ell)} & \text{otherwise.}
\end{cases}
\vspace{-0.1in}
\end{equation*}
The final sequence, $\mathbf{v}''^{(\ell)} = f_{\text{global}}(f_{\text{local}}(\mathbf{v}^{(\ell)}))$, is then fed into the attention block of layer $\ell$. 
This process allows any token to attend to summaries of all subsequent sentences by accessing the \texttt{<B-PST>} block, thereby enabling a comprehensive \emph{document-level} backward flow.

\subsection{Early Exit \& Readout}

Consistent with recent findings that intermediate layers often produce richer semantic representations~\citep{fu2024token, skean2025layer, liu-etal-2024-fantastic, jin-etal-2025-exploring}, we employ an early-exit strategy. We select the output from a predetermined intermediate layer, $L'$. The final document embedding, $\mathbf{\bar y}^{(L')}$, is then computed by applying the mean-pooling readout, as justified in \Cref{subsec:mean-vs-last}, over the fully rewired hidden states $\mathbf{v}''^{(L')}$.

\begin{table*}[ht]
\setlength{\tabcolsep}{4pt}
\centering
\resizebox{0.98\textwidth}{!}{
\begin{tabular}{llccccccccccc}
\toprule
\textbf{Models} & \textbf{Method} & \textbf{ArguAna}	& \textbf{SciFact} & \textbf{FiQA2018} &	\textbf{NFCorpus}& \textbf{SCIDOCS} &	\textbf{HotpotQA} & \textbf{Trec-COVID} & \textbf{Memory} & \textbf{Time} \\
\midrule
\multirow{6}{*}{Mistral-7B} & 
Vanilla Mean & \underline{45.66} &	\underline{42.10} &	8.02 &	9.03& 2.80 & 8.80 & 	26.02 & 1.00$\times$ & 1.00$\times$ \\
& Vanilla Last &  10.66 & 0.35 & 0.98&	2.64 &	0.27 &	0.21 &	1.96 & 1.00$\times$ & 1.00$\times$ \\
& Echo Mean & 35.99 &	28.93 &	\underline{11.60} &	\underline{13.07} & \underline{4.67} &	\underline{14.41} &	\textbf{39.13} & 4.00$\times$ & 3.45$\times$\\
& TP w. PromptEOL & 4.43 &	8.38 &	8.36 &	7.07 &	3.07 &	3.24 &	20.35 & 1.02$\times$ & 1.04$\times$  \\
& TP w. Mean & 42.41 &	36.71 &	\textbf{12.72} &	8.13 &	3.30 & 9.84 & 25.67 & 1.02$\times$ & 1.15$\times$ \\
& HTP (Ours) &  \textbf{47.06} &	\textbf{46.67} & 8.77  &		\textbf{15.51} &	\textbf{6.08} & \textbf{16.02} & \underline{33.65} &  1.12$\times$ & 1.18$\times$  \\
\midrule
\multirow{6}{*}{Gemma2-9B} 
  & Vanilla Mean         & \underline{42.70} & \underline{49.13} &  6.27 & 12.97 &  3.80 &   13.36 & 10.69    &    1.00$\times$ & 1.00$\times$   \\
& Vanilla Last          & 11.89 & 17.55 &  0.52 &  5.63 &  0.70 & 0.11 &3.59  &   1.00$\times$ & 1.00$\times$  \\
& Echo Mean             & 32.56 & 40.16 & 10.50 & 14.99 &  4.49 & \underline{20.28} & 19.42 & 4.00$\times$ & 3.16$\times$ \\
& TP w.\ PromptEOL      & 36.91 & 43.33 & \textbf{18.45} & \underline{17.66} & \underline{15.70} & \textbf{26.17} & \textbf{33.65} & 1.05$\times$ & 1.18$\times$     \\
& TP w.\ Mean           & 42.14 & 52.73 &  8.89 & 15.02 & \textbf{16.50} & 14.94 &21.22 &  1.06$\times$ & 1.16$\times$      \\
& HTP (Ours)      & \textbf{43.64} & \textbf{53.72} & \underline{10.76} &  \textbf{18.33}     &  11.16  &  18.84    &    \underline{25.16} &  1.18$\times$ & 1.20$\times$  \\
\midrule
\multirow{6}{*}{Qwen2-1.5B} 
  & Vanilla Mean         & \underline{34.11} &	\underline{27.16} &	3.26 &	3.90& 4.25 &	3.04 &	11.16 & 1.00$\times$ & 1.00$\times$  \\
& Vanilla Last          &  8.80 &	0.01 &	0.20 &	1.74 &	0.13 &	0.02	& 0.92  & 1.00$\times$ & 1.00$\times$  \\
& Echo Mean             & 33.10 & 22.98 &  \textbf{8.08} &  5.11 & \textbf{6.01} & \textbf{5.82} & 15.09 & 4.00$\times$ & 2.85$\times$       \\
& TP w.\ PromptEOL      & 16.98 &	18.54 &	\underline{4.84} &	\textbf{8.03}	& 4.81	& 4.65 &	14.72 & 1.01$\times$ & 1.02$\times$  \\
& TP w.\ Mean           &  21.66 &	18.31 &	2.52 &	\underline{5.89} &	3.91  & 5.23 &	\textbf{21.78}	& 1.01$\times$ & 1.03$\times$    \\
& Hierarchical TP       & \textbf{36.01} &	\textbf{28.29} &	4.06  &	5.43 &	\underline{4.85} & \underline{5.53} &	\underline{18.35}  & 1.10$\times$ & 1.08$\times$  \\
\bottomrule

\end{tabular}}
\caption{NDCG@10 (in percentage) on Retrieval Tasks from MTEB Retrieval Benchmarks. We \textbf{bold} the top one and \underline{underline} the runner-up. We also report the additional memory and running time incurred from Vanilla method.} \label{tab:retrieval_results}
\end{table*}

\section{Experiments}
\label{sec:experiment}
We conduct extensive experiments over HTP to answer following questions:
\begin{enumerate}[nosep]
    \item[\textbf{Q1.}] How does HTP compare with training-free LLM embedding baselines on retrieval task?
    \item[\textbf{Q2.}] Does HTP perform well in general embedding benchmarks?
    \item[\textbf{Q3.}] What is the effect of local prepending scale on retrieval performance? 
    \item[\textbf{Q4.}] Does HTP help with finetuned models?
\end{enumerate}

\subsection{BEIR Retrieval Task}
\label{sec:beir_retrieval}

\paragraph{Setups.} We evaluate on a subset of commonly used BEIR~\citep{thakur2021beir} retrieval datasets (\textbf{Q1}), including ArguAna~\citep{wachsmuth2018retrieval}, SciFact~\citep{wadden2020fact}, FiQA2018~\citep{maia201818}, NFCorpus~\citep{boteva2016full}, SCIDOCS~\citep{cohan-etal-2020-specter}, HotpotQA~\citep{yang-etal-2018-hotpotqa}, and TREC-COVID~\citep{voorhees2021trec}. 
We report NDCG@10, the time, and memory cost for each baseline methods.

\paragraph{Baselines.} We evaluate three decoder-only LLMs of varying sizes: (1) \textit{Mistral-Instruct-7B-0.3}~\citep{jiang2023mistral7b}, (2) \textit{Gemma2-9B}~\citep{gemmateam2024gemma2improvingopen}, and (3) \textit{Qwen2-1.5B-Instruct}~\citep{yang2024qwen2technicalreport}. For each model, we benchmark the following embedding extraction strategies: (1) \textit{Vanilla Mean}: mean pooling over all token embeddings; (2) \textit{Vanilla Last}: using the last token’s embedding; (3) \textit{Echo Mean}~\citep{echoembedding}: duplicating the input and averaging token embeddings from the second pass; (4) \textit{Token Prepending (TP)}~\citep{fu2024token} \textit{w. PromptEOL}~\citep{prompteol}: appending a summarization-in-one-word prompt and using the final token’s embedding; (5) \textit{TP w. Mean}: prepended prompt token with mean pooling; and (6) \textit{HTP}: as detailed in \Cref{sec:method}. For HTP, we use $K=1$ across all datasets. To segment the paragraphs, we use Spacy's parser \cite{honnibal2020spacy}. Except for TP w. PromptEOL, all methods incorporate instructions (details in \cref{sec:experiment_details}). 

For TP methods and HTP, we select TP mixing layers and early exiting layers based on embedding performance of separate validation retrieval datasets from BEIR \cite{thakur2021beir}, following previous practice \cite{fu2024token} (See details about early existing and token prepending layers in \cref{app:llm_arch}). 

\paragraph{Results.} \cref{tab:retrieval_results} presents the overall results. To answer \textbf{Q1}, our HTP method demonstrates comparable or even superior performance to other training-free LLM embedding baselines, frequently ranking first or second across the datasets. In terms of memory and time efficiency, HTP requires significantly less than Echo Mean~\citep{echoembedding}, while achieving similar performance. Compared to Token Prepending (TP)~\citep{fu2024token} using either PromptEOL~\citep{prompteol} or mean pooling, HTP shows slightly better performance. This improvement stems from its design, which involves inserting additional $\texttt{<PST>}_m$ and $\texttt{<B-PST>}_m$ tokens for local and global lookups, leading to reduced information squashing. In terms of models, \emph{Gemma2-9B} achieves stronger performance, likely due to its larger scale. Overall, we observe that using the last token embedding generally results in lower retrieval performance compared to mean pooling, with the exception of \emph{Gemma2-9B}'s TP with PromptEOL. This suggests that PromptEOL~\citep{prompteol} acts as a soft prompting technique that implicitly encourages information reorganization within the representation.

\begin{figure*}[ht]
    \centering
    \begin{minipage}{0.31\textwidth}
        \centering
        \includegraphics[width=0.9\linewidth]{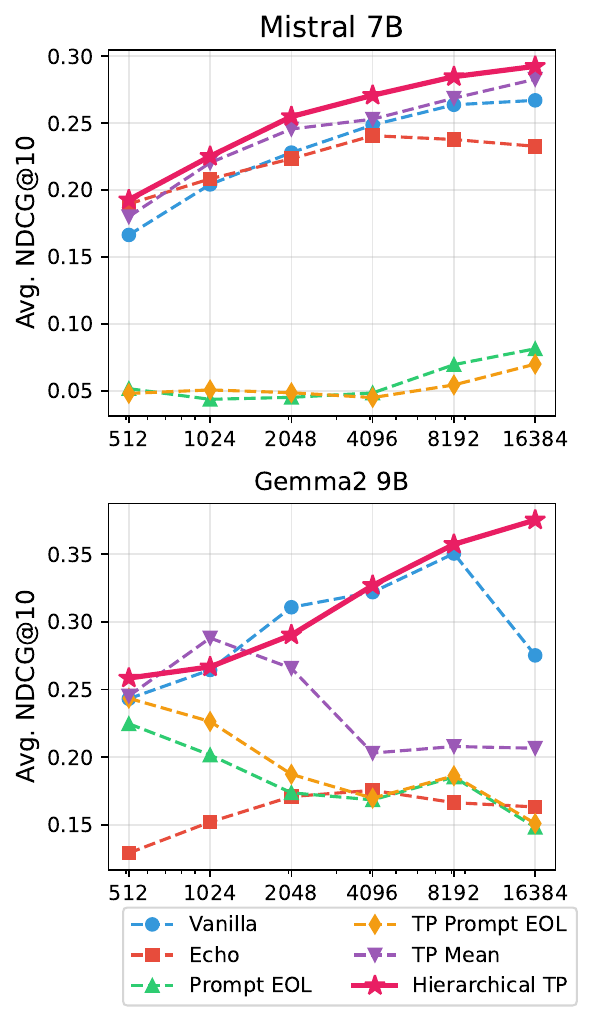} 
        \caption{Avg. NDCG@10 across different context lengths.}
        \label{fig:longembed_plot}
    \end{minipage}%
    \hfill%
    \begin{minipage}{0.66\textwidth}
        \centering
        \resizebox{0.98
        \linewidth}{!}{
        \begin{tabular}{llccccccc}
        \toprule
        \textbf{Models} & \textbf{Method} & \textbf{CXT Len} & \textbf{QMSum} & \textbf{2WikiMQA} & \textbf{SumFD} & \textbf{NQA}  & \textbf{Avg. Time(s)} \\
        \midrule
        \multirow{6}{*}{Mistral-7B} &  
        Vanilla Mean & 512 & 11.98 & 15.29 & 36.81 & 2.50 & 73.0 (\textcolor{gray}{1.00$\times$}) \\
        & Echo Mean & 512 & \textbf{14.27} & \textbf{22.50} & 32.03 &  \underline{7.05} & 135.0 (\textcolor{gray}{1.84$\times$}) \\
        & PromptEOL & 512 & 4.57 & 7.16 & 6.07 & 2.80 & 90.5 (\textcolor{gray}{1.25$\times$}) \\
        & TP w. PromptEOL & 512 & 5.44 & 6.51 & 5.52 & 1.72 & 94.5 (\textcolor{gray}{1.28$\times$})\\
        & TP w. Mean &  512 & 11.97 & 18.62 & \textbf{38.50} & 3.08 & 93.7 (\textcolor{gray}{1.27$\times$}) 	 \\
        & HTP (Ours) & 512 & \underline{13.22} & \underline{20.17} & \underline{35.63} & \textbf{8.09} & 98.2 (\textcolor{gray}{1.34$\times$})  \\
        \midrule
        \multirow{6}{*}{Gemma2-9B} & Vanilla Mean & 512 & 19.19 & \underline{23.12} & \underline{48.54} & 6.45 & 86.5 (\textcolor{gray}{1.00$\times$}) \\
        & Echo Mean & 512 & 9.73 & 16.30 & 21.54 & 4.14 & 170.5 (\textcolor{gray}{2.01$\times$}) \\
        & PromptEOL  & 512 &13.90 & 27.35 & 32.90 & \underline{15.74} & 132.5 (\textcolor{gray}{1.56$\times$})\\
        & TP w. PromptEOL & 512 & 14.43 & \textbf{30.86} & 36.16 & \textbf{15.89} & 118.3 (\textcolor{gray}{1.40$\times$}) \\
        & TP w. Mean  & 512 & \underline{20.45} & 21.73 & \textbf{48.88} & 7.04 & 103.0 (\textcolor{gray}{1.22$\times$}) \\
        & HTP (Ours) & 512 & \textbf{21.89} & 22.90 & 47.78 & 10.84 & 107.7 (\textcolor{gray}{1.27$\times$}) \\
        \midrule
        \multirow{6}{*}{Mistral-7B} &  
        Vanilla Mean & 8192 &24.18& 24.02 & \underline{53.88} & 3.42 & 236.7 (\textcolor{gray}{1.00$\times$}) \\
        & Echo Mean & 8192 & 17.29 & \textbf{29.78} & 39.27 & \underline{8.77} & 818.5 (\textcolor{gray}{3.50$\times$}) \\
        & PromptEOL & 8192 & 6.36 & 7.49 & 9.86 & 4.08 & 245.5 (\textcolor{gray}{1.46$\times$})\\
        & TP w. PromptEOL & 8192 & 4.56 & 6.89 & 6.86 & 3.42 & 256.5 (\textcolor{gray}{1.05$\times$})\\
        & TP w. Mean &  8192 & \underline{23.08} & 23.11 & \textbf{56.50} & 4.78 & 241.0 (\textcolor{gray}{1.02$\times$}) \\
        & HTP (Ours) & 8192 & \textbf{23.43} & \underline{27.88} & 53.53 & \textbf{9.07} & 265.9 (\textcolor{gray}{1.50$\times$}) \\
        \midrule
        \multirow{6}{*}{Gemma2-9B} & Vanilla Mean & 8192 & \underline{29.85} &	\underline{34.62} &	\underline{66.69} &	7.02 & 322.5 (\textcolor{gray}{1.00$\times$}) \\
        & Echo Mean & 8192 & 10.23 &	17.15	& 28.52	& 5.76 & 832.1 (\textcolor{gray}{2.58$\times$})  \\
        & PromptEOL & 8192 & 9.66 &	21.84 &	30.74 &	\textbf{11.91} & 322.7 (\textcolor{gray}{1.00$\times$})\\
        & TP w. PromptEOL & 8192 & 9.71	& 23.32 & 29.80 &	\underline{11.70} & 336.3 (\textcolor{gray}{1.03$\times$}) \\
        & TP w. Mean & 8192 & 18.29 & 19.20 &	29.22 & 7.21 & 	337.5 (\textcolor{gray}{1.03$\times$}) \\
        & HTP (Ours) & 8192  & \textbf{30.22} &	\textbf{35.19} &	\textbf{67.06} &	10.42 & 350.4 (\textcolor{gray}{1.08$\times$})  \\
        \bottomrule
        \end{tabular}
}
        \captionof{table}{NDCG@10 (in percentage) and avg. running time on LongEmbed. we report the context length of 512 and an extended length of 8192. }\label{tab:longembed_results}
    \end{minipage}
\end{figure*}

\subsection{LongEmbed Retrieval Tasks}

\paragraph{Setups.} We evaluate performance of models on four real-world tasks in LongEmbed~\citep{zhu-etal-2024-longembed}, which features documents of longer length and dispersed target information (\textbf{Q1}). The four tasks are QMSum~\citep{zhong2021qmsum}, 2WikiMultiHopQA~\citep{ho-etal-2020-constructing}, SummScreenFD~\citep{Chen2021SummScreenAD}, NarrativeQA~\citep{kovcisky2018narrativeqa}, and a detailed description of the dataset is in \Cref{tab:longembed-stats}. We evaluate on \textit{Gemma2-9B} and \textit{Mistral-instruct-7B-0.3}, and over the same extraction strategies as in the previous section. We again report NDCG@10 and running time for the evaluation. We primarily use two context lengths: 512, a commonly used length in the embedding model literature \citep{echoembedding,lee2025nvembed,behnamghader2024llmvec}, and an extended length of 8192. We also evaluate models with additional context lengths and report their performance for comparison. Across all models, we fix one single instruction prompt: \emph{``Retrieve relevant document. \{text\}''} (except for PromptEOL \cite{prompteol} where the prompt is fixed to \emph{``The paragraph \{text\} means in one word:''}). For our HTP model, we simply use $K=1$ across all dataets.

\paragraph{Results.} \cref{tab:longembed_results} presents results for context lengths of 512 and 8192, while \cref{fig:longembed_plot} shows the average NDCG@10 across four datasets for a wider range of context lengths. At both context lengths, HTP achieves strong performance. It notably outperforms other methods at 8192, while incurring lower runtime than the most expensive Echo Embedding \citep{echoembedding} (\textbf{Q1}). Furthermore, it shows improved performance with longer context lengths, likely due to local token prepending preserving more local information and reducing information oversquashing. \cref{fig:longembed_plot} shows that, across both \emph{Mistral} and \emph{Gemma2} models, the top-performing methods (HTP, Vanilla, TP Mean) rely on mean embedding, supporting our earlier claim that it captures more information and offers greater stability for longer contexts.

\subsection{General Embedding Tasks}

\paragraph{Setups.} Beyond retrieval tasks, we assess the embeddings' quality on a wider range of downstream tasks from the Massive Text Embedding Benchmark (MTEB)~\citep{muennighoff2023mteb}. We evaluate across six task categories: Classification (11 datasets), Reranking (3 datasets), Clustering (11 datasets), 
and Semantic Textual Similarity (STS) (5 datasets), and BEIR Retrieval tasks (7 datasets) (\textbf{Q2}). We report the average performance for each category using \textit{Mistral-Instruct-7B-0.3}, and compare HTP with Echo Mean, PromptEOL, and TP with PromptEOL, which are described in detail in \Cref{sec:beir_retrieval}. For the Echo Mean and HTP methods, we prepend commonly used instruction-style prompts to the texts (except for STS), formatted as \emph{``\{instruct\} \{text\}''}. For the other two methods, we use the same PromptEOL~\citep{prompteol} prompt: \emph{``The sentence means \{text\} in one word.''} Details are provided in  \cref{app:general_embedding_tasks}.

\begin{table}[t]
\footnotesize
\setlength{\tabcolsep}{4pt}
\centering
\scalebox{0.9}{
\begin{tabular}{@{}lcccc@{}}
\toprule
\multirow{2}{*}{\textbf{Task}} & \textbf{Echo} & \textbf{Prompt-} & \textbf{TP w.} & \textbf{HTP} \\
& \textbf{Mean} & \textbf{EOL} & \textbf{ PromptEOL} & \textbf{(Ours)} \\
\midrule
Cls (11 datasets) & 66.12 & 68.12 & \textbf{69.52} & \underline{68.44} \\
Rerank (3 datasets) & \textbf{43.31} & 40.35 & 40.57 & \underline{40.85} \\
Cluster (11 datasets) & \underline{34.07} & 23.10 & 22.11 & \textbf{35.39} \\
STS (5 datasets) & 64.46 & \underline{67.89} & \textbf{68.46} & 53.64 \\
Retrieval (7 datasets) & \underline{21.11} & 8.15 & 7.84 & \textbf{24.82} \\
\bottomrule
\end{tabular}
}
\caption{Average performance on general embedding tasks. We \textbf{bold} the top one and \underline{underline} the runner-up.} \label{tab:nlp_tasks_summary_transposed}
\end{table}

\paragraph{Results.} We present the results in \cref{tab:nlp_tasks_summary_transposed}. To address \textbf{Q2}, we observe that HTP performs well on classification, reranking, and clustering tasks. 
However, it lags behind other methods on sentence similarity tasks. 
We attribute this to the fact that these tasks generally do not involve long sequences,
as STS tasks focus on intra-sentence similarity. Hence, considering PromptEOL \cite{prompteol} excels by summarizing each sentence with a single representative token, it is much more suited for such fine-grained comparisons. Individual task results are in \cref{app:general_embedding_result}. 

\subsection{Ablations over Local Prepending Strategy}
\label{sec:ablation}

\paragraph{Effect of Hyperparameter $K$.} 
The granularity of the input partitions, controlled by the hyperparameter $K$ (the number of sentences per summary block), presents a crucial trade-off between summary detail and coherence. To investigate HTP's sensitivity to the scale of local prepending (\textbf{Q3}), we ablate $K$ on two short-document retrieval datasets (NFCorpus, SciFact with maximum context length of 512), and two long-document datasets (SummScreenFD, 2WikiMultipleQA with maximum context length of 16,382).

As shown in \Cref{fig:vary-K-htp}, the results reveal distinct trends based on context length. For short documents, performance degrades as $K$ increases, suggesting that coarser summaries may elide critical, fine-grained details necessary for the task. Conversely, for long documents, performance improves with a larger $K$. We attribute this to two factors: (1) larger sentence blocks yield more coherent semantic summaries in lengthy, multi-topic contexts, and (2) a very small $K$ in a long document creates an excessive number of \texttt{<B-PST>} tokens, which can push the model into out-of-distribution behavior. This suggests the optimal partitioning strategy is dependent on document length and task granularity.

\begin{figure}
    \centering
    \includegraphics[width=0.92\linewidth]{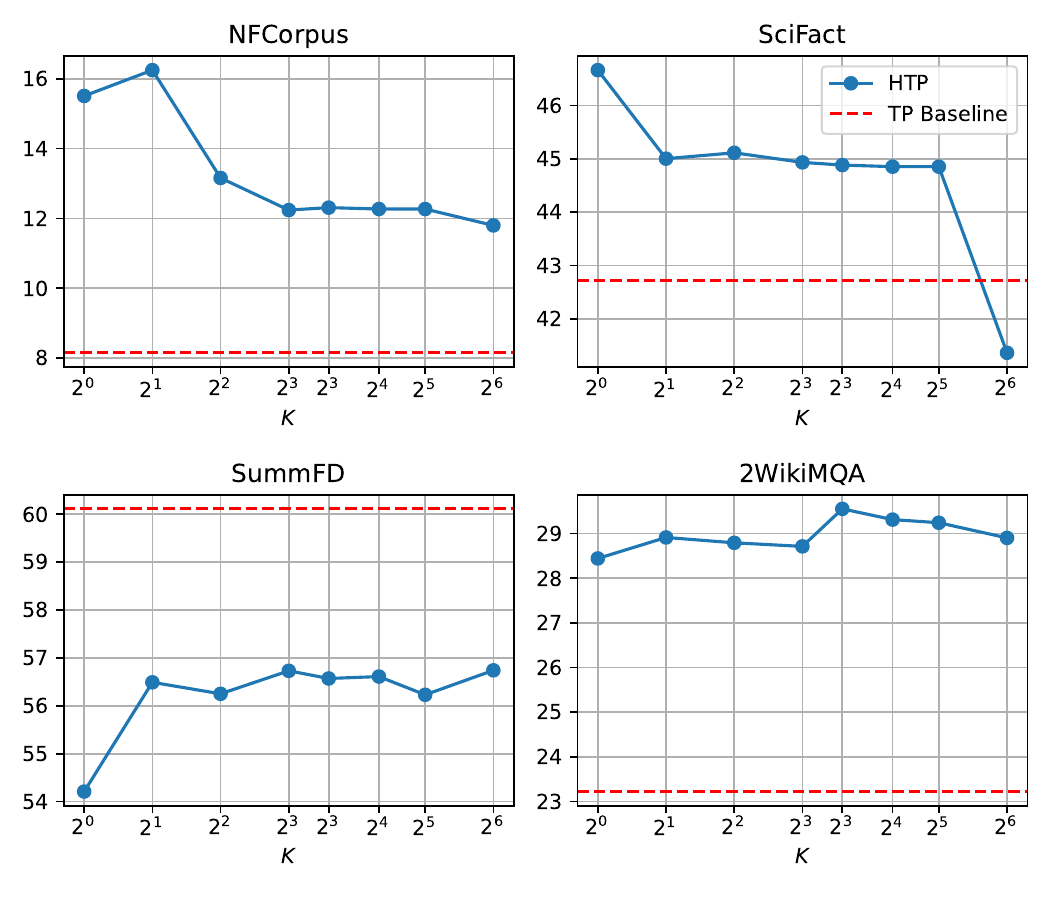}
    \vspace{-0.2in}
    \caption{
    Performance of HTP with various $K$.}
    \vspace{-0.18in}
    \label{fig:vary-K-htp}
\end{figure}

\paragraph{Sentences vs. Every Few tokens.} 
\label{sec:sentence_vs_token} 
We also conduct an ablation to validate our choice of partitioning the input along semantic boundaries (i.e., sentences) rather than at arbitrary fixed-length intervals. We compare our standard approach against a baseline that inserts \texttt{<PST>} tokens every $N$ tokens. To ensure a fair comparison, we dynamically set $N$ for each document to match the average sentence length in tokens, thereby keeping the total number of inserted \texttt{<PST>} tokens identical between the two methods (i.e., $N = \left\lfloor \frac{\texttt{num\_tokens}}{\texttt{n\_sentence}} \right\rfloor$).

The results, shown in \cref{fig:sentence-htp} for the NFCorpus and SciFact datasets, demonstrate that sentence-based partitioning consistently outperforms the fixed-interval baseline. This supports our hypothesis that leveraging natural linguistic boundaries allows the final token of a sentence to capture a more meaningful and coherent semantic summary, ultimately leading to higher-quality representations.

\subsection{HTP on Finetuned Embedding Models} 
Finally, we answer \textbf{(Q4)} by modifying over \emph{NV-Embed-v2}~\citep{lee2025nvembed}, a finetuned model for embedding tasks, and show that HTP can yield more performance. \emph
{NV-Embed-v2} is a general-purpose embedding model based on \emph{Mistral-7B}, enhanced with bi-directional attention and a novel attention aggregation mechanism in the final layer. As a result, we do not enable early-layer extraction or mean pooling as in the previous HTP setup. All other hyperparameter configurations remain the same as described in \Cref{app:llm_arch}.
\Cref{tab:nv_embed_perf} presents the results of \emph{NV-Embed-v2} with existing training-free methods. We observe that HTP consistently boosts performance over the base \emph{NV-Embed-v2} model across all three datasets. This result suggests that the gains from HTP are orthogonal to those achieved through finetuning. Furthermore, HTP demonstrates highly competitive performance against the other training-free methods, demonstrating its robust utility.

\begin{figure}[t]
    \centering
    \includegraphics[width=0.92\linewidth]{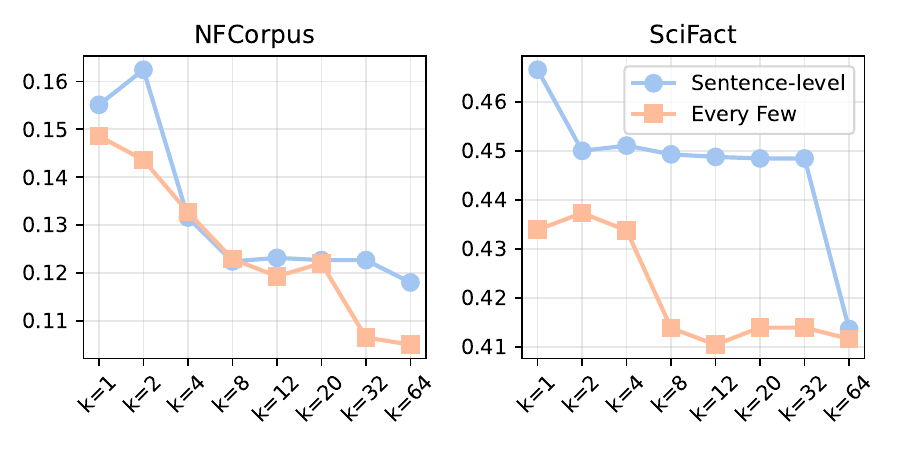}
    \vspace{-0.1in}
    \caption{Comparison of Sentence vs. Every Few tokens.}
    \label{fig:sentence-htp}
    \vspace{-0.1in}
\end{figure}

\begin{table}[t]
\centering
\resizebox{0.42\textwidth}{!}{
\begin{tabular}{lccc}
\toprule
\textbf{Model} & \textbf{NFCorpus} & \textbf{FiQA2018} & \textbf{SciFact} \\
\midrule
NV-Embed      & 45.10 & 62.63 & 80.92 \\
\; + Echo & 40.88 & 48.35 & 74.08 \\
\; + TP Mean &  \underline{46.60} & \underline{62.68} & \underline{81.64} \\
\; + HTP  & \textbf{47.26} & \textbf{64.18} & \textbf{82.52} \\
\bottomrule
\end{tabular}}
\caption{
 Performance on NV-Embed. We \textbf{bold} the top one and \underline{underline} the runner-up.}
\label{tab:nv_embed_perf}
\vspace{-0.15in}
\end{table}

\section{Conclusion}
\label{sec:conclusion}
We introduced \emph{Hierarchical Token Prepending} (HTP), a simple, training-free method that resolves critical information bottlenecks in LLM embeddings. Our approach is directly informed by theoretical and empirical analysis that pinpoints these specific failure modes. By partitioning the input and creating a hierarchy of block-level summary tokens, HTP establishes multiple backward information pathways, mitigating attention-level compression. Paired with a robust mean-pooling readout, HTP achieves consistent performance gains across extensive benchmarks, especially in long-context settings, offering a scalable approach to transforming powerful generative models into superior universal text encoders. Its ability to enhance both zero-shot and already finetuned models underscores its utility as a versatile tool for future document retrieval and analysis systems.

\section*{Limitations}
While HTP significantly improves zero-shot embeddings from decoder-only LLMs, its performance is not expected to surpass models that are extensively finetuned specifically for retrieval tasks. Our initial experiments showed that HTP can enhance an existing finetuned model (NV-Embed~\cite{lee2025nvembed}), but a broader investigation is needed to fully understand its interaction with diverse model architectures and training paradigms. 
A deeper look into the mechanisms of backward dependency and token prepending is also warranted, which we leave for future work.

\bibliography{reference}

@inproceedings{fu2024token,
  title={Token prepending: A training-free approach for eliciting better sentence embeddings from llms},
  author={Fu, Yuchen and Cheng, Zifeng and Jiang, Zhiwei and Wang, Zhonghui and Yin, Yafeng and Li, Zhengliang and Gu, Qing},
  booktitle={ACL},
  year={2024}
}

@inproceedings{echoembedding,
  title={Repetition Improves Language Model Embeddings},
  author={Springer, Jacob Mitchell and Kotha, Suhas and Fried, Daniel and Neubig, Graham and Raghunathan, Aditi},
  booktitle={ICLR},
  year={2025}
}

@inproceedings{devlin2019bert,
  title={Bert: Pre-training of deep bidirectional transformers for language understanding},
  author={Devlin, Jacob and Chang, Ming-Wei and Lee, Kenton and Toutanova, Kristina},
  booktitle={NAACL},
  year={2019}
}

@article{liu2019roberta,
  title={Roberta: A robustly optimized bert pretraining approach},
  author={Liu, Yinhan and Ott, Myle and Goyal, Naman and Du, Jingfei and Joshi, Mandar and Chen, Danqi and Levy, Omer and Lewis, Mike and Zettlemoyer, Luke and Stoyanov, Veselin},
  journal={arXiv preprint arXiv:1907.11692},
  year={2019}
}

@inproceedings{
barbero2024transformers,
title={Transformers need glasses! Information over-squashing in language tasks},
author={Federico Barbero and Andrea Banino and Steven Kapturowski and Dharshan Kumaran and Jo{\~a}o Guilherme Madeira Ara{\'u}jo and Alex Vitvitskyi and Razvan Pascanu and Petar Veli{\v{c}}kovi{\'c}},
booktitle={NeurIPS},
year={2024},
}

@inproceedings{muennighoff2023mteb,
  title={MTEB: Massive Text Embedding Benchmark},
  author={Muennighoff, Niklas and Tazi, Nouamane and Magne, Loic and Reimers, Nils},
  booktitle={EACL},
  year={2023}
}

@inproceedings{
thakur2021beir,
title={{BEIR}: A Heterogeneous Benchmark for Zero-shot Evaluation of Information Retrieval Models},
author={Nandan Thakur and Nils Reimers and Andreas R{\"u}ckl{\'e} and Abhishek Srivastava and Iryna Gurevych},
booktitle={NeurIPS Datasets and Benchmarks Track},
year={2021},
}

@inproceedings{xu-etal-2024-reading,
    title = "Re-Reading Improves Reasoning in Large Language Models",
    author = "Xu, Xiaohan  and
      Tao, Chongyang  and
      Shen, Tao  and
      Xu, Can  and
      Xu, Hongbo  and
      Long, Guodong  and
      Lou, Jian-Guang  and
      Ma, Shuai",
    booktitle = "EMNLP",
    year = "2024",
}

@inproceedings{finetunellama,
author = {Ma, Xueguang and Wang, Liang and Yang, Nan and Wei, Furu and Lin, Jimmy},
title = {Fine-Tuning LLaMA for Multi-Stage Text Retrieval},
year = {2024},
booktitle = {SIGIR},
}

@inproceedings{li-li-2024-bellm,
    title = "{B}e{LLM}: Backward Dependency Enhanced Large Language Model for Sentence Embeddings",
    author = "Li, Xianming  and
      Li, Jing",
    booktitle = "NAACL",
    year = "2024",
}

@inproceedings{
behnamghader2024llmvec,
title={{LLM}2Vec: Large Language Models Are Secretly Powerful Text Encoders},
author={Parishad BehnamGhader and Vaibhav Adlakha and Marius Mosbach and Dzmitry Bahdanau and Nicolas Chapados and Siva Reddy},
booktitle={CoLM},
year={2024},
}

@inproceedings{
muennighoff2025generative,
title={Generative Representational Instruction Tuning},
author={Niklas Muennighoff and Hongjin SU and Liang Wang and Nan Yang and Furu Wei and Tao Yu and Amanpreet Singh and Douwe Kiela},
booktitle={ICLR},
year={2025},
}

@inproceedings{prompteol,
    title = "Scaling Sentence Embeddings with Large Language Models",
    author = "Jiang, Ting  and
      Huang, Shaohan  and
      Luan, Zhongzhi  and
      Wang, Deqing  and
      Zhuang, Fuzhen",
    booktitle = "EMNLP",
    year = "2024",
}

@inproceedings{lei-etal-2024-meta,
    title = "Meta-Task Prompting Elicits Embeddings from Large Language Models",
    author = "Lei, Yibin  and
      Wu, Di  and
      Zhou, Tianyi  and
      Shen, Tao  and
      Cao, Yu  and
      Tao, Chongyang  and
      Yates, Andrew",
    booktitle = "ACL",
    month = aug,
    year = "2024",
}

@inproceedings{
li2025your,
title={Your Mixture-of-Experts {LLM} Is Secretly an Embedding Model for Free},
author={Ziyue Li and Tianyi Zhou},
booktitle={ICLR},
year={2025},
}

@inproceedings{pretended_cot,
author = {Zhang, Bowen and Chang, Kehua and Li, Chunping},
title = {Simple Techniques for Enhancing Sentence Embeddings in Generative Language Models},
year = {2024},
booktitle = {ICIC},
}

@article{cheng2025contrastive,
  title={Contrastive Prompting Enhances Sentence Embeddings in LLMs through Inference-Time Steering},
  author={Cheng, Zifeng and Wang, Zhonghui and Fu, Yuchen and Jiang, Zhiwei and Yin, Yafeng and Wang, Cong and Gu, Qing},
  journal={arXiv preprint arXiv:2505.12831},
  year={2025}
}

@inproceedings{gao2021simcse,
   title={{SimCSE}: Simple Contrastive Learning of Sentence Embeddings},
   author={Gao, Tianyu and Yao, Xingcheng and Chen, Danqi},
   booktitle={EMNLP},
   year={2021}
}

@inproceedings{reimers-2019-sentence-bert,
    title = "Sentence-BERT: Sentence Embeddings using Siamese BERT-Networks",
    author = "Reimers, Nils and Gurevych, Iryna",
    booktitle = "EMNLP",
    year = "2019",
}

@inproceedings{
lee2025nvembed,
title={{NV}-Embed: Improved Techniques for Training {LLM}s as Generalist Embedding Models},
author={Chankyu Lee and Rajarshi Roy and Mengyao Xu and Jonathan Raiman and Mohammad Shoeybi and Bryan Catanzaro and Wei Ping},
booktitle={ICLR},
year={2025},
}

@inproceedings{jiang-etal-2022-promptbert,
    title = "{P}rompt{BERT}: Improving {BERT} Sentence Embeddings with Prompts",
    author = "Jiang, Ting  and
      Jiao, Jian  and
      Huang, Shaohan  and
      Zhang, Zihan  and
      Wang, Deqing  and
      Zhuang, Fuzhen  and
      Wei, Furu  and
      Huang, Haizhen  and
      Deng, Denvy  and
      Zhang, Qi",
    booktitle = "EMNLP",
    year = "2022",
}

@inproceedings{ni-etal-2022-sentence,
    title = "Sentence-T5: Scalable Sentence Encoders from Pre-trained Text-to-Text Models",
    author = "Ni, Jianmo  and
      Hernandez Abrego, Gustavo  and
      Constant, Noah  and
      Ma, Ji  and
      Hall, Keith  and
      Cer, Daniel  and
      Yang, Yinfei",
    booktitle = "ACL Findings",
    month = may,
    year = "2022",
}

@inproceedings{chanchani-huang-2023-composition,
    title = "Composition-contrastive Learning for Sentence Embeddings",
    author = "Chanchani, Sachin  and
      Huang, Ruihong",
    booktitle = "ACL",
    month = jul,
    year = "2023",
}

@inproceedings{
topping2022understanding,
title={Understanding over-squashing and bottlenecks on graphs via curvature},
author={Jake Topping and Francesco Di Giovanni and Benjamin Paul Chamberlain and Xiaowen Dong and Michael M. Bronstein},
booktitle={ICLR},
year={2022},
}

@misc{ba2016layernormalization,
      title={Layer Normalization}, 
      author={Jimmy Lei Ba and Jamie Ryan Kiros and Geoffrey E. Hinton},
      year={2016},
      eprint={1607.06450},
      archivePrefix={arXiv},
      primaryClass={stat.ML},
}

@article{rope,
author = {Su, Jianlin and Ahmed, Murtadha and Lu, Yu and Pan, Shengfeng and Bo, Wen and Liu, Yunfeng},
title = {RoFormer: Enhanced transformer with Rotary Position Embedding},
year = {2024},
journal = {Neurocomput.},
}

@article{honnibal2020spacy,
  title = {spaCy: Industrial-strength Natural Language Processing in Python},
  author = {Honnibal, Matthew and Montani, Ines and Van Landeghem, Sofie and Boyd, Adriane},
  year = {2020},
  doi = {10.5281/zenodo.1212303},
}

@inproceedings{wachsmuth2018retrieval,
  title={Retrieval of the best counterargument without prior topic knowledge},
  author={Wachsmuth, Henning and Syed, Shahbaz and Stein, Benno},
  booktitle={ACL},
  year={2018}
}

@inproceedings{wadden2020fact,
  title={Fact or Fiction: Verifying Scientific Claims},
  author={Wadden, David and Lin, Shanchuan and Lo, Kyle and Wang, Lucy Lu and van Zuylen, Madeleine and Cohan, Arman and Hajishirzi, Hannaneh},
  booktitle={EMNLP},
  year={2020}
}

@inproceedings{maia201818,
  title={Www'18 open challenge: financial opinion mining and question answering},
  author={Maia, Macedo and Handschuh, Siegfried and Freitas, Andr{\'e} and Davis, Brian and McDermott, Ross and Zarrouk, Manel and Balahur, Alexandra},
  booktitle={Companion proceedings of the the web conference},
  year={2018}
}

@inproceedings{boteva2016full,
  title={A full-text learning to rank dataset for medical information retrieval},
  author={Boteva, Vera and Gholipour, Demian and Sokolov, Artem and Riezler, Stefan},
  booktitle={European Conference on Information Retrieval},
  year={2016},
}

@inproceedings{cohan-etal-2020-specter,
    title = "{SPECTER}: Document-level Representation Learning using Citation-informed Transformers",
    author = "Cohan, Arman  and
      Feldman, Sergey  and
      Beltagy, Iz  and
      Downey, Doug  and
      Weld, Daniel",
    booktitle = "ACL",
    year = "2020",
}

@inproceedings{yang-etal-2018-hotpotqa,
    title = "{H}otpot{QA}: A Dataset for Diverse, Explainable Multi-hop Question Answering",
    author = "Yang, Zhilin  and
      Qi, Peng  and
      Zhang, Saizheng  and
      Bengio, Yoshua  and
      Cohen, William  and
      Salakhutdinov, Ruslan  and
      Manning, Christopher D.",
    booktitle = "EMNLP",
    year = "2018",
}

@inproceedings{voorhees2021trec,
  title={TREC-COVID: constructing a pandemic information retrieval test collection},
  author={Voorhees, Ellen and Alam, Tasmeer and Bedrick, Steven and Demner-Fushman, Dina and Hersh, William R and Lo, Kyle and Roberts, Kirk and Soboroff, Ian and Wang, Lucy Lu},
  booktitle={ACM SIGIR Forum},
  year={2021},
}

@misc{gemmateam2024gemma2improvingopen,
      title={Gemma 2: Improving Open Language Models at a Practical Size}, 
      author={Gemma Team and Morgane Riviere and Shreya Pathak and Pier Giuseppe Sessa and Cassidy Hardin and Surya Bhupatiraju and Léonard Hussenot and Thomas Mesnard and Bobak Shahriari and Alexandre Ramé and Johan Ferret and Peter Liu and Pouya Tafti and Abe Friesen and Michelle Casbon and Sabela Ramos and Ravin Kumar and Charline Le Lan and Sammy Jerome and Anton Tsitsulin and Nino Vieillard and Piotr Stanczyk and Sertan Girgin and Nikola Momchev and Matt Hoffman and Shantanu Thakoor and Jean-Bastien Grill and Behnam Neyshabur and Olivier Bachem and Alanna Walton and Aliaksei Severyn and Alicia Parrish and Aliya Ahmad and Allen Hutchison and Alvin Abdagic and Amanda Carl and Amy Shen and Andy Brock and Andy Coenen and Anthony Laforge and Antonia Paterson and Ben Bastian and Bilal Piot and Bo Wu and Brandon Royal and Charlie Chen and Chintu Kumar and Chris Perry and Chris Welty and Christopher A. Choquette-Choo and Danila Sinopalnikov and David Weinberger and Dimple Vijaykumar and Dominika Rogozińska and Dustin Herbison and Elisa Bandy and Emma Wang and Eric Noland and Erica Moreira and Evan Senter and Evgenii Eltyshev and Francesco Visin and Gabriel Rasskin and Gary Wei and Glenn Cameron and Gus Martins and Hadi Hashemi and Hanna Klimczak-Plucińska and Harleen Batra and Harsh Dhand and Ivan Nardini and Jacinda Mein and Jack Zhou and James Svensson and Jeff Stanway and Jetha Chan and Jin Peng Zhou and Joana Carrasqueira and Joana Iljazi and Jocelyn Becker and Joe Fernandez and Joost van Amersfoort and Josh Gordon and Josh Lipschultz and Josh Newlan and Ju-yeong Ji and Kareem Mohamed and Kartikeya Badola and Kat Black and Katie Millican and Keelin McDonell and Kelvin Nguyen and Kiranbir Sodhia and Kish Greene and Lars Lowe Sjoesund and Lauren Usui and Laurent Sifre and Lena Heuermann and Leticia Lago and Lilly McNealus and Livio Baldini Soares and Logan Kilpatrick and Lucas Dixon and Luciano Martins and Machel Reid and Manvinder Singh and Mark Iverson and Martin Görner and Mat Velloso and Mateo Wirth and Matt Davidow and Matt Miller and Matthew Rahtz and Matthew Watson and Meg Risdal and Mehran Kazemi and Michael Moynihan and Ming Zhang and Minsuk Kahng and Minwoo Park and Mofi Rahman and Mohit Khatwani and Natalie Dao and Nenshad Bardoliwalla and Nesh Devanathan and Neta Dumai and Nilay Chauhan and Oscar Wahltinez and Pankil Botarda and Parker Barnes and Paul Barham and Paul Michel and Pengchong Jin and Petko Georgiev and Phil Culliton and Pradeep Kuppala and Ramona Comanescu and Ramona Merhej and Reena Jana and Reza Ardeshir Rokni and Rishabh Agarwal and Ryan Mullins and Samaneh Saadat and Sara Mc Carthy and Sarah Cogan and Sarah Perrin and Sébastien M. R. Arnold and Sebastian Krause and Shengyang Dai and Shruti Garg and Shruti Sheth and Sue Ronstrom and Susan Chan and Timothy Jordan and Ting Yu and Tom Eccles and Tom Hennigan and Tomas Kocisky and Tulsee Doshi and Vihan Jain and Vikas Yadav and Vilobh Meshram and Vishal Dharmadhikari and Warren Barkley and Wei Wei and Wenming Ye and Woohyun Han and Woosuk Kwon and Xiang Xu and Zhe Shen and Zhitao Gong and Zichuan Wei and Victor Cotruta and Phoebe Kirk and Anand Rao and Minh Giang and Ludovic Peran and Tris Warkentin and Eli Collins and Joelle Barral and Zoubin Ghahramani and Raia Hadsell and D. Sculley and Jeanine Banks and Anca Dragan and Slav Petrov and Oriol Vinyals and Jeff Dean and Demis Hassabis and Koray Kavukcuoglu and Clement Farabet and Elena Buchatskaya and Sebastian Borgeaud and Noah Fiedel and Armand Joulin and Kathleen Kenealy and Robert Dadashi and Alek Andreev},
      year={2024},
      eprint={2408.00118},
      archivePrefix={arXiv},
      primaryClass={cs.CL},
      url={https://arxiv.org/abs/2408.00118}, 
}

@misc{jiang2023mistral7b,
      title={Mistral 7B}, 
      author={Albert Q. Jiang and Alexandre Sablayrolles and Arthur Mensch and Chris Bamford and Devendra Singh Chaplot and Diego de las Casas and Florian Bressand and Gianna Lengyel and Guillaume Lample and Lucile Saulnier and Lélio Renard Lavaud and Marie-Anne Lachaux and Pierre Stock and Teven Le Scao and Thibaut Lavril and Thomas Wang and Timothée Lacroix and William El Sayed},
      year={2023},
      eprint={2310.06825},
      archivePrefix={arXiv},
      primaryClass={cs.CL},
      url={https://arxiv.org/abs/2310.06825}, 
}

@misc{yang2024qwen2technicalreport,
      title={Qwen2 Technical Report}, 
      author={An Yang and Baosong Yang and Binyuan Hui and Bo Zheng and Bowen Yu and Chang Zhou and Chengpeng Li and Chengyuan Li and Dayiheng Liu and Fei Huang and Guanting Dong and Haoran Wei and Huan Lin and Jialong Tang and Jialin Wang and Jian Yang and Jianhong Tu and Jianwei Zhang and Jianxin Ma and Jianxin Yang and Jin Xu and Jingren Zhou and Jinze Bai and Jinzheng He and Junyang Lin and Kai Dang and Keming Lu and Keqin Chen and Kexin Yang and Mei Li and Mingfeng Xue and Na Ni and Pei Zhang and Peng Wang and Ru Peng and Rui Men and Ruize Gao and Runji Lin and Shijie Wang and Shuai Bai and Sinan Tan and Tianhang Zhu and Tianhao Li and Tianyu Liu and Wenbin Ge and Xiaodong Deng and Xiaohuan Zhou and Xingzhang Ren and Xinyu Zhang and Xipin Wei and Xuancheng Ren and Xuejing Liu and Yang Fan and Yang Yao and Yichang Zhang and Yu Wan and Yunfei Chu and Yuqiong Liu and Zeyu Cui and Zhenru Zhang and Zhifang Guo and Zhihao Fan},
      year={2024},
      eprint={2407.10671},
      archivePrefix={arXiv},
      primaryClass={cs.CL},
      url={https://arxiv.org/abs/2407.10671}, 
}

@inproceedings{zhu-etal-2024-longembed,
    title = "{L}ong{E}mbed: Extending Embedding Models for Long Context Retrieval",
    author = "Zhu, Dawei  and
      Wang, Liang  and
      Yang, Nan  and
      Song, Yifan  and
      Wu, Wenhao  and
      Wei, Furu  and
      Li, Sujian",
    booktitle = "EMNLP",
    year = "2024",
}

@inproceedings{zhong2021qmsum,
   title={{QMS}um: {A} {N}ew {B}enchmark for {Q}uery-based {M}ulti-domain {M}eeting {S}ummarization},
   author={Zhong, Ming and Yin, Da and Yu, Tao and Zaidi, Ahmad and Mutuma, Mutethia and Jha, Rahul and Hassan Awadallah, Ahmed and Celikyilmaz, Asli and Liu, Yang and Qiu, Xipeng and Radev, Dragomir},
   booktitle={NAACL},
   year={2021}
}

@inproceedings{ho-etal-2020-constructing,
    title = "Constructing A Multi-hop {QA} Dataset for Comprehensive Evaluation of Reasoning Steps",
    author = "Ho, Xanh  and
      Duong Nguyen, Anh-Khoa  and
      Sugawara, Saku  and
      Aizawa, Akiko",
    booktitle = "Proceedings of the 28th International Conference on Computational Linguistics",
    year = "2020",
}

@inproceedings{Chen2021SummScreenAD,
  title={SummScreen: A Dataset for Abstractive Screenplay Summarization},
  author={Mingda Chen and Zewei Chu and Sam Wiseman and Kevin Gimpel},
  booktitle={ACL},
  year={2021},
}

@article{kovcisky2018narrativeqa,
  title={The narrativeqa reading comprehension challenge},
  author={Ko{\v{c}}isk{\`y}, Tom{\'a}{\v{s}} and Schwarz, Jonathan and Blunsom, Phil and Dyer, Chris and Hermann, Karl Moritz and Melis, G{\'a}bor and Grefenstette, Edward},
  journal={Transactions of the Association for Computational Linguistics},
  year={2018},
}

@inproceedings{liu-etal-2024-fantastic,
    title = "Fantastic Semantics and Where to Find Them: Investigating Which Layers of Generative {LLM}s Reflect Lexical Semantics",
    author = "Liu, Zhu  and
      Kong, Cunliang  and
      Liu, Ying  and
      Sun, Maosong",
    booktitle = "Findings of ACL",
    year = "2024",

}

@inproceedings{jin-etal-2025-exploring,
    title = "Exploring Concept Depth: How Large Language Models Acquire Knowledge and Concept at Different Layers?",
    author = "Jin, Mingyu  and Yu, Qinkai  and
      Huang, Jingyuan  and
      Zeng, Qingcheng  and
      Wang, Zhenting  and
      Hua, Wenyue  and
      Zhao, Haiyan  and
      Mei, Kai  and
      Meng, Yanda  and
      Ding, Kaize  and
      Yang, Fan  and
      Du, Mengnan  and
      Zhang, Yongfeng",
    booktitle = "Proceedings of the 31st International Conference on Computational Linguistics",
    year = "2025",
}

@inproceedings{
skean2025layer,
title={Layer by Layer: Uncovering Hidden Representations in Language Models},
author={Oscar Skean and Md Rifat Arefin and Dan Zhao and Niket Nikul Patel and Jalal Naghiyev and Yann LeCun and Ravid Shwartz-Ziv},
booktitle={ICML},
year={2025},
}

\appendix
\onecolumn

\section{Proofs}
\label[appendix]{sec:proofs}

We follow \citet{barbero2024transformers} and study the over-squashing effect of the readout. In particular, we conduct \emph{sensitivity analysis} original introduced in the context of graph learning~\citep{topping2022understanding}, which aim to study the quantity $\partial \textbf{y}_n/\partial \textbf{v}_i^{(0)}$, i.e., the partial derivative of final output embedding (last token) with respect to the $i$-th token. We first define the studied decoder-only transformer architectures, which encompass the majority of the LLMs currently in use.

\subsection{Transformer architectures}
\label[appendix]{sec:transformers}

\paragraph{Notation.}
We write $[n]=\{1,\dots,n\}$. We use the Euclidean norm $\|\cdot\|$ for vectors and the induced operator (spectral) norm for Jacobians. $\delta_{ij}$ represents \emph{Kronecker delta}: $\delta_{ij}=1$ if $i=j$ and $0$ otherwise. We denote the $i$‑th standard basis vector by $\mathbf{e}_i$ and the all‑ones vector by $\mathbf{1}$.

\paragraph{Setup.} We follow the model specification used by \citet{barbero2024transformers}, and study decoder‑only, causal Transformers of $d$ dimension on a length‑$n$ sequence of token states
$\mathbf{v}^{(0)}=(\mathbf{v}^{(0)}_1,\dots,\mathbf{v}^{(0)}_n)$ with $\mathbf{v}^{(0)}_i\in\mathbb{R}^d$. We let $\mathbf{Q}, \mathbf{K}, \mathbf{V} \in \mathbb{R}^{n \times d}$ be the query, key, and value matrices for a sequence of $n$ tokens with $d$-dimensional embeddings. We denote the $i$-th token’s query, key, and value vectors by $\mathbf{q}_i, \mathbf{k}_i, \mathbf{v}_i \in \mathbb{R}^{d}$.
Let $\mathbf{p}_{ij}\in\mathbb{R}^{2e}$ be the $2e$ dimensional vectors encoding positional information between positions $i$ and $j$.
We assume $\sup_{i,j}\|\mathbf{p}_{ij}\|\le P_{\max}<\infty$ (bounded positional encodings), which holds for most positional schemes used in practice, such as Rotational positional encoding (RoPE)~\citep{rope}.

\paragraph{Single‑head Pre‑LN block.}
For layer $\ell\in\{0,\dots,L-1\}$ and position $i\in[n]$, define
\begin{align}
\mathbf{z}_i^{(\ell)}
  &= \mathbf{v}_i^{(\ell)} \;+\; \sum_{j \le i} \alpha_{ij}^{(\ell)} \,\mathrm{norm}^{(\ell)}_1\!\bigl( \mathbf{v}_j^{(\ell)} \bigr),
  \tag{Attn}\label{eq:attn-update}\\
\alpha_{ij}^{(\ell)}
  &= \frac{\exp\!\Big(k\!\big(\mathbf{q}_i^{(\ell)},\, \mathbf{k}_j^{(\ell)},\, \mathbf{p}_{ij}\big)\Big)}
           {\sum_{w \le i} \exp\!\Big(k\!\big(\mathbf{q}_i^{(\ell)},\, \mathbf{k}_w^{(\ell)},\, \mathbf{p}_{iw}\big)\Big)},
  \quad j\le i,\qquad \alpha_{ij}^{(\ell)}=0 \text{ if } j>i,
  \tag{Softmax}\label{eq:attn-weights}\\
\mathbf{v}_i^{(\ell+1)}
  &= \mathbf{z}_i^{(\ell)} \;+\; \psi^{(\ell)}\!\Big(\mathrm{norm}^{(\ell)}_2\!\big(\mathbf{z}_i^{(\ell)}\big)\Big),
  \tag{MLP}\label{eq:mlp-update}
\end{align}

where $k:\mathbb{R}^d\times\mathbb{R}^d\times\mathbb{R}^{2e}\to\mathbb{R}$ is a scoring function (e.g.\ a bilinear form with a positional bias), $\psi^{(\ell)}:\mathbb{R}^d\to\mathbb{R}^d$ is the MLP at layer $\ell$, and $\mathrm{norm}^{(\ell)}_1,\mathrm{norm}^{(\ell)}_2:\mathbb{R}^d\to\mathbb{R}^d$ are the (pre‑activation) normalization maps (typically LayerNorm~\citep{ba2016layernormalization}).
Causality is enforced by the mask $j\le i$ in \eqref{eq:attn-update}–\eqref{eq:attn-weights}, hence $\mathbf{v}^{(\ell)}_j$ depends only on $\{\mathbf{v}^{(\ell-1)}_i:i\le j\}$.

\paragraph{Attention matrix.}
Let $\boldsymbol{\Lambda}^{(\ell)}\in\mathbb{R}^{n\times n}$ collect attention weights with $(\boldsymbol{\Lambda}^{(\ell)})_{ij}=\alpha^{(\ell)}_{ij}$.
Each $\boldsymbol{\Lambda}^{(\ell)}$ is \emph{row–stochastic} ($\boldsymbol{\Lambda}^{(\ell)}\mathbf{1}=\mathbf{1}$) and \emph{lower–triangular} ($( \boldsymbol{\Lambda}^{(\ell)} )_{ij}=0$ if $j>i$).
This lower–triangular, row–stochastic structure is preserved under products.%
See Lemmas B.6–B.7 in \citet{barbero2024transformers}. %

\paragraph{Final normalization and readouts.}
After layer $L$, we set
\[
\mathbf{y}_i \;=\; \mathrm{norm}_3\!\bigl(\mathbf{v}_i^{(L)}\bigr), \qquad i \in [n].
\]
where $\mathrm{norm}_3:\mathbb{R}^d\to\mathbb{R}^d$ is a normalization (often LayerNorm).
For last-token embedding, the readout usually uses the \emph{last token} $\mathbf{y}_n$; for mean‑pooling, we set $\bar{\mathbf{y}}=\frac1n\sum_{i=1}^n \mathbf{y}_i$.

\paragraph{Assumptions.}
Following \citet{barbero2024transformers}, we adopt the following simplifications for layerwise sensitivity bounds. We write $\operatorname{Lipschitz}(f)$ for the Lipschitz constant of a map $f$, i.e., any $L$ such that
$$\|f(x)-f(y)\|\le L\,\|x-y\|,\qquad \forall x,y.$$
\begin{enumerate}[leftmargin=*,itemsep=0pt,topsep=2pt]
\item During differentiation, the attention weights $\alpha_{ij}^{(\ell)}$ are treated as input–independent constants.
\item Each normalization operator has a Lipschitz bound with known scalings, with
\begin{align*}
\operatorname{Lipschitz}\!\big(\mathrm{norm}_1^{(\ell)}\big)\;\le\;\frac{1}{\beta_1^{(\ell)}},\quad
\operatorname{Lipschitz}\!\big(\mathrm{norm}_2^{(\ell)}\big)\;\le\;\frac{1}{\beta_2^{(\ell)}},\quad
\operatorname{Lipschitz}\!\big(\mathrm{norm}_3\big)\;\le\;\frac{1}{\beta_3},
\end{align*}
for some $\beta_1^{(\ell)},\beta_2^{(\ell)},\beta_3>0$.
\item The MLP in each layer $\psi^{(\ell)}$ admits a  Lipschitz constant $\sigma_{\psi}^{(\ell)}$.
\end{enumerate}

\begin{remark}[Multi‑head and projections]
The statements below naturally extend to $H$ heads by stacking the single‑head updates in parallel and absorbing the output projection into the constants $\sigma_\psi^{(\ell)}$ and $\beta_2^{(\ell)}$; all proofs go through with the same structure because the causal, residual, and row‑stochastic properties are preserved.
\end{remark}

We now proceed to show the layerwise sensitivity bounds. First, we note that we follow a very similar proof idea shown in \citet{barbero2024transformers}: in fact, the layerwise bounds and the last-token embeddings are exactly the ones shown in Theorem 5.1 and Theorem B.5 of \citet{barbero2024transformers} under mild modification. We include the results nonetheless for the sake of completeness, which aids the derivation of mean-token readouts. 

\subsection{Layerwise bounds and path expansion}

The residual form \eqref{eq:attn-update}–\eqref{eq:mlp-update} and our Lipschitz bounds imply, for $j\ge i$,

\begin{align}
\Big\|\frac{\partial \mathbf{v}_j^{(\ell+1)}}{\partial \mathbf{v}_i^{(\ell)}}\Big\|
&=
\Big\|\frac{\partial}{\partial \mathbf{v}_i^{(\ell)}}\Big[\psi^{(\ell)}\!\big(\mathrm{norm}_2^{(\ell)}(\mathbf{z}_j^{(\ell)})\big)
+\mathbf{z}_j^{(\ell)}\Big]\Big\| \\
&\le \Big(\tfrac{\sigma_\psi^{(\ell)}}{\beta_2^{(\ell)}}+1\Big)
\Big\|\frac{\partial \mathbf{z}_j^{(\ell)}}{\partial \mathbf{v}_i^{(\ell)}}\Big\| \\
&=\Big(\tfrac{\sigma_\psi^{(\ell)}}{\beta_2^{(\ell)}}+1\Big)
\Big\|\frac{\partial}{\partial \mathbf{v}_i^{(\ell)}}\Big[\sum_{m\le j}\alpha^{(\ell)}_{j,m}\,
\mathrm{norm}_1^{(\ell)}\!\big(\mathbf{v}_m^{(\ell)}\big)+\mathbf{v}_j^{(\ell)}\Big]\Big\| \\
&\le \Big(\tfrac{\sigma_\psi^{(\ell)}}{\beta_2^{(\ell)}}+1\Big)
\Big(\tfrac{\alpha^{(\ell)}_{j,i}}{\beta_1^{(\ell)}}+\delta_{j,i}\Big), \tag{$\ast$}
\label{eq:one-layer-bound}
\end{align}

and 
$\frac{\partial \mathbf{v}_{j}^{(\ell+1)}}{\partial \mathbf{v}_i^{(\ell)}}=0$ if $j<i$.

Define the \emph{residual‑augmented attention} $\mathbf{\bar\alpha}^{(\ell)}\in\mathbb{R}^{n\times n}$ by
\[
\mathbf{\bar\alpha}^{(\ell)}_{j,i}:=\frac{\alpha^{(\ell)}_{j,i}}{\beta_1^{(\ell)}}+\delta_{j,i}
\quad\text{for }i\le j,
\qquad
\mathbf{\bar\alpha}^{(\ell)}_{j,i}=0\text{ for }i>j,
\]
and let the row‑sum be $r_\ell:=\sum_{i\le j}\mathbf{\bar\alpha}^{(\ell)}_{j,i}=1+\tfrac{1}{\beta_1^{(\ell)}}$ (because $\sum_{i\le j}\alpha^{(\ell)}_{j,i}=1$). Normalize
\[
\mathbf{M}^{(\ell)}:=\frac{1}{r_\ell}\,\mathbf{\bar\alpha}^{(\ell)}
\quad\Longrightarrow\quad 
\mathbf{M}^{(\ell)}\mathbf{1}=\mathbf{1},
\]
i.e., $\mathbf{M}^{(\ell)}$ is lower‑triangular and row‑stochastic.

\begin{lemma}[Path–sum equals matrix entry]\label{lem:path-sum-matrix}
Let $\mathbf{\bar\alpha}^{(\ell)}$ and $\mathbf{M}^{(\ell)}=\frac{1}{r_\ell}\mathbf{\bar\alpha}^{(\ell)}$ be as above, and
$\mathbf{A}:=\mathbf{M}^{(L-1)}\cdots \mathbf{M}^{(0)}$. For any $i\le j$,
\[
\sum_{k_1\ge i}\sum_{k_2\ge k_1}\cdots\sum_{k_L\ge k_{L-1}}
\bar{\alpha}^{(L-1)}_{j,k_L}\!\!\prod_{\ell=2}^{L-1}\bar{\alpha}^{(\ell-1)}_{k_\ell,k_{\ell-1}}\,
\bar{\alpha}^{(0)}_{k_1,i}
\;=\;
\Big(\prod_{\ell=0}^{L-1} r_\ell\Big)\,\mathbf{A}_{j,i}.
\]
\end{lemma}

\begin{proof}
Write the left-hand side as the $(j,i)$ entry of the matrix product
$\mathbf{\bar\alpha}^{(L-1)}\cdots \mathbf{\bar\alpha}^{(0)}$:
by definition of matrix multiplication for lower-triangular matrices,
\[
\big(\mathbf{\bar\alpha}^{(L-1)}\cdots \mathbf{\bar\alpha}^{(0)}\big)_{j,i}
=\sum_{k_{L}\ge \cdots \ge k_1}\mathbf{\bar\alpha}^{(L-1)}_{j,k_L}\cdots \mathbf{\bar\alpha}^{(0)}_{k_1,i},
\]
where the index constraints $k_\ell\ge k_{\ell-1}$ (and $j\ge k_L$, $k_1\ge i$) are exactly those enforced by lower-triangularity.
Now factor each layer’s row-sum: $\mathbf{\bar\alpha}^{(\ell)}=r_\ell \mathbf{M}^{(\ell)}$.
Hence
\[
\mathbf{\bar\alpha}^{(L-1)}\cdots \mathbf{\bar\alpha}^{(0)}
=\Big(\prod_{\ell=0}^{L-1}r_\ell\Big)\, \mathbf{M}^{(L-1)}\cdots \mathbf{M}^{(0)}
=\Big(\prod_{\ell=0}^{L-1}r_\ell\Big)\,\mathbf{A}.
\]
Taking the $(j,i)$ entry gives the claim. The nested sums are precisely the $(j,i)$ entry of the product $\mathbf{\bar\alpha}^{(L-1)}\cdots\mathbf{\bar\alpha}^{(0)}$ because lower-triangularity imposes the constraints $j\ge k_L\ge\cdots\ge k_1\ge i$.

\end{proof}

\subsection{Oversquashing bounds for last‑token and mean‑token readouts}

We are now ready to present the main results. Note that the following results for the last token directly correlate to Theorem B.5, \citet{barbero2024transformers}.

\begin{theorem}[Last‑token vs mean‑token sensitivity]\label{thm:last-vs-mean}
Under the assumptions above, let
\[
C \;:=\; \frac{1}{\beta_3}\,\prod_{\ell=0}^{L-1}\Big(\frac{\sigma_\psi^{(\ell)}}{\beta_2^{(\ell)}}+1\Big),
\qquad
K_L \;:=\; C\,\prod_{\ell=0}^{L-1} r_\ell,
\qquad
\mathbf{A} \;:=\; \mathbf{M}^{(L-1)}\cdots \mathbf{M}^{(0)}.
\]
Then for every input position $i\in[n]$:
\begin{align*}
\textbf{(a) Last token:}\quad
&\Big\|\frac{\mathbf{\partial y}_n}{\mathbf{\partial v}_i^{(0)}}\Big\| \;\le\; K_L\,\mathbf{A}_{n,i}.\\[2pt]
\textbf{(b) Mean pooling:}\quad
&\Big\|\frac{\mathbf{\partial \bar y}}{\mathbf{\partial v}_i^{(0)}}\Big\| \;\le\; \frac{K_L}{n}\sum_{j=1}^n\mathbf{A}_{j,i}. 
\end{align*}
\end{theorem}

\begin{proof}
By the chain rule across layers, $\mathbf{\partial v}_n^{(L)}/\mathbf{\partial v}_i^{(0)}$ equals a sum over causal paths $i\!\to\!k_1\!\to\!\cdots\!\to\!k_L\!\to\!n$ with one Jacobian factor per layer. Following the proof strategies in Theorem B.5, \citet{barbero2024transformers}  and using \eqref{eq:one-layer-bound} at each layer and $\|\partial \mathbf{y}_n/\partial \mathbf{v}_n^{(L)}\|\le 1/\beta_3$ gives
\[
\Big\|\frac{\mathbf{\partial y}_n}{\mathbf{\partial v}_i^{(0)}}\Big\|
\;\le\;
\frac{1}{\beta_3}\!\prod_{\ell=0}^{L-1}\!\Big(\frac{\sigma_\psi^{(\ell)}}{\beta_2^{(\ell)}}+1\Big)
\!\!\sum_{k_1\ge i}\cdots\sum_{k_L\ge k_{L-1}}\!
\bar{\alpha}^{(L-1)}_{n,k_L}\!\!\prod_{\ell=2}^{L-1}\bar{\alpha}^{(\ell-1)}_{k_\ell,k_{\ell-1}}\,
\bar{\alpha}^{(0)}_{k_1,i}.
\]
Lemma~\ref{lem:path-sum-matrix} converts the multi‑sum to $\big(\prod_\ell r_\ell\big)\mathbf{A}_{n,i}$, yielding (a).
For (b), by linearity:
\[
\frac{\partial \mathbf{\bar y}}{\mathbf{\partial v}_i^{(0)}}=\frac{1}{n}\sum_{j=1}^n\frac{\partial \mathbf{y}_j}{\mathbf{\partial v}_i^{(0)}}
\quad\Rightarrow\quad
\Big\|\frac{\partial \mathbf{\bar y}}{\mathbf{\partial v}_i^{(0)}}\Big\|
\le \frac{K_L}{n}\sum_{j=1}^n\mathbf{A}_{j,i}.
\]

Here, the term $\sum_{j=1}^n \mathbf{A}_{j,i}$ is the sum of the $i$-th column of $\mathbf{A}$.

\end{proof}

\paragraph{Interpretation and Discussion}

Our sensitivity bounds reveal both a depth-dependent growth factor and a structural transport term. The growth factor
\[
K_L \;=\;
\underbrace{\frac{1}{\beta_3}}_{\text{final norm}}
\underbrace{\prod_{\ell=0}^{L-1}\Big(\tfrac{\sigma_\psi^{(\ell)}}{\beta_2^{(\ell)}}+1\Big)}_{\text{MLP+residual per layer}}
\underbrace{\prod_{\ell=0}^{L-1}\Big(1+\tfrac{1}{\beta_1^{(\ell)}}\Big)}_{\text{attn+residual row sums}}
\]
typically grows (often exponentially) with depth \(L\), scaling the magnitude of the bound. Orthogonal to this, the structural term $\mathbf{A}=\mathbf{M}^{(L-1)}\cdots \mathbf{M}^{(0)}$
 governs how signal moves and attenuates across layers, and the choice of readout changes \emph{which} part of \(\mathbf{A}\) the bound depends on.

With \emph{last-token readout}, the sensitivity is controlled by a \textbf{single entry} \(\mathbf{A}_{n,i}\), i.e., the influence that reaches one fixed “sink” \(n\) from source \(i\). Random-walk intuition suggests that a single entry like $\mathbf{A}_{n,i}$ can shrink rapidly with depth. In the homogeneous left-drifting regime of \citet{barbero2024transformers}, Proposition~B.8, one even has $\mathbf{A}=\mathbf{M}^L\to \mathbf{1}e_1^\top$, hence $\mathbf{A}_{n,i}\to 0$ for $i>1$.

In contrast, \emph{mean-pooling} depends on the \textbf{entire column sum} \(\sum_j \mathbf{A}_{j,i}\), which aggregates influence delivered from \(i\) to \emph{all} outputs. This aggregates total outgoing mass rather than betting on a single path to a single sink, and thus is structurally more robust: it does not suffer the same guaranteed decay as an individual matrix entry. Consequently, while \(K_L\) sets the overall scale of the bound, the readout choice determines whether the structural term induces decay (last-token) or preserves signal (mean-pooling), explaining why mean-pooling provably mitigates over-squashing of early tokens.

\twocolumn

\section{Experiment details}
\label[appendix]{sec:experiment_details}

\subsection{Retrieval tasks}
\label[appendix]{app:general_embedding_tasks}
The detailed description of the statistics of the BEIR evaluation dataset can be found in \cite{thakur2021beir}. We show the characteristics of the datasets of the real-world subsets of the LongEmbed\cite{zhu-etal-2024-longembed} datasets in \cref{tab:longembed-stats}.

\subsubsection{LLM Architecture Setup}
\label[appendix]{app:llm_arch}
\cref{tab:tp_model_configs} shows the LLM configuration for TP-based and HTP models. We briefly describe the version of LLM used, the starting and ending layers of the TP methods (following \cite{fu2024token}), and the early exit layer.

\begin{table}[t]
\centering
\footnotesize
\scalebox{0.9}{
\begin{tabular}{@{}ll@{}}
\toprule
\textbf{Model} & \textbf{Configuration Details} \\
\midrule
\texttt{mistral-instruct} &
\begin{tabular}[t]{@{}l@{}}
Mistral-7B-Instruct-v0.3 \\
TP plan applied from layer 1 to 7 \\
Uses output from third to last layer \\
\end{tabular} \\

\texttt{gemma-2-9b} &
\begin{tabular}[t]{@{}l@{}}
Gemma-2-9b \\
TP plan applied from layer 1 to 6 \\
Uses output from second to last layer \\
\end{tabular} \\

\texttt{qwen2-instruct} &
\begin{tabular}[t]{@{}l@{}}
Qwen2.5-1.5B-Instruct \\
TP plan applied from layer 1 to 7 \\
Uses output from second to last layer \\
\end{tabular} \\
\bottomrule
\end{tabular} 
}
\caption{Summary of TP-based and HTP model configurations.} \label{tab:tp_model_configs}
\end{table}

\subsubsection{Instruction For Retrievals}
\cref{tab:retrieval_instructions} shows the instructions used in retrieval tasks for both BEIR \cite{thakur2021beir} and LongEmbed \cite{zhu-etal-2024-longembed}.

\begin{table*}[t]
\footnotesize
\centering
\begin{tabular}{@{}ll@{}}
\toprule
\textbf{Task Name} & \textbf{Instruction Template} \\
\midrule
ArguAna & Given a claim, retrieve documents that support or refute the claim \\
FiQA2018 & Given a financial question, retrieve user replies that best answer the question \\
HotpotQA & Given a multi-hop question, retrieve documents that can help answer the question \\
NFCorpus & Given a question, retrieve relevant documents that answer the question \\
SCIDOCS & Given a scientific paper title, retrieve paper abstracts that are cited by the given paper \\
SciFact & Given a scientific claim, retrieve documents that support or refute the claim \\
TREC-COVID & Given a query on COVID-19, retrieve documents that answer the query \\
\midrule
NarrativeQA      &  Retrieve the relevant document \\
QMSum           & Retrieve the relevant document \\
2WikiMultihopQA  & Retrieve the relevant document \\
SummScreenFD     & Retrieve the relevant document \\
\bottomrule
\end{tabular}
\caption{Instructions used for evaluation on the BEIR benchmark and LongEmbed benchmark.} \label{tab:retrieval_instructions}
\end{table*}

\subsection{General embedding tasks}
\label[appendix]{app:general_embedding_tasks}

To evaluate the generalization capabilities of HTP and baseline methods, we benchmark on a diverse set of 30 public datasets from the Massive Text Embedding Benchmark (MTEB)~\citep{muennighoff2023mteb}. We report the average performance across four categories of tasks: Classification, Reranking, Clustering, and Semantic Textual Similarity (STS).

The specific datasets used in our evaluation are as follows:
\begin{itemize}

    \item \textbf{Classification} (11 datasets): We use accuracy as the metric for AmazonCounterfactual, AmazonReview, Banking77, Emotion, Imdb, MassiveIntent, MassiveScenario, MTOPDomain, MTOPIntent, ToxicConversations, and TweetSentiment.

    \item \textbf{Reranking} (3 datasets): We use Mean Average Precision (MAP) for AskUbuntuDupQuestions, MindSmallReranking, and StackOverflowDupQuestions.

    \item \textbf{Clustering} (11 datasets): We use the V-measure score for ArxivClusteringP2P, ArxivClusteringS2S, BiorxivClusteringP2P, BiorxivClusteringS2S, MedrxivClusteringP2P, MedrxivClusteringS2S, RedditClustering,RedditClusteringP2P,
StackExchangeClustering,
StackExchangeClusteringP2P and
TwentyNewsgroupsClustering.


    \item \textbf{STS} (5 datasets): We report Spearman correlation for the standard STS12 through STS16 benchmarks.

\end{itemize}

\subsubsection{Experiment Setup}
\label[appendix]{app:general_embedding_exp}

\cref{tab:mteb_instructions} shows the instructions (prompts) used in acquiring the general embeddings.

\begin{table*}[t]
\footnotesize
\centering
\begin{tabular}{@{}lp{0.6\linewidth}}
\toprule
\textbf{Task Name} & \textbf{Instruction Template} \\
\midrule
AmazonCounterfactualClassification & Classify a given Amazon customer review text as either counterfactual or non-counterfactual \\
AmazonPolarityClassification & Classify Amazon reviews into positive or negative sentiment \\
AmazonReviewsClassification & Classify the given Amazon review into its appropriate rating category \\
Banking77Classification & Given an online banking query, find the corresponding intents \\
EmotionClassification & Classify the emotion expressed in the given Twitter message into one of the six emotions: anger, fear, joy, love, sadness, and surprise \\
ImdbClassification & Classify the sentiment expressed in the given movie review text from the IMDB dataset \\
MassiveIntentClassification & Given a user utterance as query, find the user intents \\
MassiveScenarioClassification & Given a user utterance as query, find the user scenarios \\
MTOPDomainClassification & Classify the intent domain of the given utterance in task-oriented conversation \\
MTOPIntentClassification & Classify the intent of the given utterance in task-oriented conversation \\
ToxicConversationsClassification & Classify the given comments as either toxic or not toxic \\
TweetSentimentExtractionClassification & Classify the sentiment of a given tweet as either positive, negative, or neutral \\
\midrule
ArxivClusteringP2P & Identify the main and secondary category of Arxiv papers based on the titles and abstracts \\
ArxivClusteringS2S & Identify the main and secondary category of Arxiv papers based on the titles \\
BiorxivClusteringP2P & Identify the main category of Biorxiv papers based on the titles and abstracts \\
BiorxivClusteringS2S & Identify the main category of Biorxiv papers based on the titles \\
MedrxivClusteringP2P & Identify the main category of Medrxiv papers based on the titles and abstracts \\
MedrxivClusteringS2S & Identify the main category of Medrxiv papers based on the titles \\
RedditClustering & Identify the topic of the given Reddit posts based on the titles \\
RedditClusteringP2P & Identify the topic of the Reddit posts based on the titles and posts \\
StackExchangeClustering & Identify the topic or theme of StackExchange posts based on the titles \\
StackExchangeClusteringP2P & Identify the topic or theme of the StackExchange posts based on the given paragraphs \\
TwentyNewsgroupsClustering & Identify the topic or theme of the given news articles \\
\midrule
AskUbuntuDupQuestions & Retrieve duplicate questions from AskUbuntu forum \\
MindSmallReranking & Retrieve relevant news articles based on user browsing history \\
StackOverflowDupQuestions & Retrieve duplicate questions from StackOverflow forum \\
\bottomrule
\end{tabular}
\caption{Prompts used for evaluation on the general embedding tasks. } \label{tab:mteb_instructions}
\end{table*}

\subsubsection{General Embedding Result}
\label[appendix]{app:general_embedding_result}

We show the general embedding performance for Classification tasks in \cref{tab:nlp_tasks_classification}, Reranking tasks in \cref{tab:nlp_tasks_reranking}, Clustering tasks in \cref{tab:nlp_tasks_cluster},
and STS tasks in \cref{tab:nlp_tasks_sts}.

\begin{table}[t]
\footnotesize
\setlength{\tabcolsep}{4pt}
\centering
\scalebox{0.9}{
\begin{tabular}{@{}lcccc@{}}
\toprule
\textbf{Task} & \textbf{Echo} & \textbf{Prompt-} & \textbf{TP w.} & \textbf{HTP} \\
& \textbf{Mean} & \textbf{EOL} & \textbf{PromptEOL} & \textbf{(Ours)} \\
\midrule
AmazonCounterfactual     & 70.36 & 73.88 & 74.09 & 78.64 \\
Banking77                & 77.50 & 66.46 & 66.46 & 68.88 \\
Emotion                  & 38.90 & 46.44 & 46.43 & 47.51 \\
IMDb                     & 74.17 & 80.20 & 90.20 & 78.35 \\
MassiveIntent            & 70.61 & 71.71 & 67.25 & 75.20 \\
MassiveScenario          & 75.42 & 72.88 & 69.78 & 76.58 \\
MTOPDomain               & 88.77 & 87.19 & 87.19 & 85.39 \\
MTOPIntent               & 76.28 & 82.43 & 82.43 & 78.05 \\
ToxicConversations       & 65.49 & 73.13 & 73.13 & 73.14 \\
TweetSentiment           & 50.04 & 48.44 & 61.21 & 49.26 \\
AmazonReview             & 39.76 & 46.56 & 46.56 & 41.84 \\
\midrule
\textbf{Average}         & 66.12 & 68.12 & \textbf{69.52} & \underline{68.44} \\
\bottomrule
\end{tabular}
}
\caption{Performance on classification tasks (\%). We \textbf{bold} the top one and \underline{underline} the runner-up in the average row.} \label{tab:nlp_tasks_classification}
\end{table}

\begin{table}[t]
\footnotesize
\setlength{\tabcolsep}{4pt}
\centering
\scalebox{0.80}{
\begin{tabular}{@{}lcccc@{}}
\toprule
\textbf{Task} & \textbf{Echo} & \textbf{Prompt-} & \textbf{TP w.} & \textbf{HTP} \\
& \textbf{Mean} & \textbf{EOL} & \textbf{PromptEOL} & \textbf{(Ours)} \\
\midrule
AskUbuntuDupQuestions     & 56.09 & 53.65 & 53.65 & 51.05 \\
MindSmallReranking        & 28.60 & 27.29 & 27.84 & 28.96 \\
StackOverflowDupQuestions & 45.23 & 40.11 & 40.21 & 42.53 \\
\midrule
\textbf{Average}          & \textbf{43.31} & 40.35 & 40.57 & \underline{40.85} \\
\bottomrule
\end{tabular}
}
\caption{Performance on reranking tasks (\%). We \textbf{bold} the top one and \underline{underline} the runner-up in the average row.} \label{tab:nlp_tasks_reranking}
\end{table}

\begin{table}[t]
\footnotesize
\setlength{\tabcolsep}{4pt}
\centering
\scalebox{0.8}{
\begin{tabular}{@{}lcccc@{}}
\toprule
\textbf{Task} & \textbf{Echo} & \textbf{Prompt-} & \textbf{TP w.} & \textbf{HTP} \\
& \textbf{Mean} & \textbf{EOL} & \textbf{PromptEOL} & \textbf{(Ours)} \\
\midrule
ArxivClusteringP2P        & 43.18 & 30.98 & 27.46 & 48.48 \\
ArxivClusteringS2S        & 37.13 & 25.01 & 24.70 & 34.50 \\
BiorxivClusteringP2P      & 31.93 & 17.91 & 17.53 & 37.74 \\
BiorxivClusteringS2S      & 25.28 & 14.12 & 13.86 & 26.23 \\
MedrxivClusteringP2P      & 27.10 & 17.50 & 17.04 & 30.36 \\
MedrxivClusteringS2S      & 23.86 & 17.93 & 17.28 & 25.65 \\
RedditClustering          & 36.05 & 16.55 & 15.48 & 26.78 \\
RedditClusteringP2P       & 56.10 & 34.41 & 33.92 & 59.20 \\
StackExchangeClustering   & 43.11 & 30.27 & 28.88 & 43.53 \\
StackExchangeClusteringP2P& 36.50 & 26.17 & 25.64 & 35.31 \\
TwentyNewsgroupsClustering& 21.60 & 23.29 & 21.44 & 21.46 \\
\midrule
\textbf{Average}          & \underline{34.71} & 23.10 & 22.11 & \textbf{35.39} \\
\bottomrule
\end{tabular}
}
\caption{Average performance on clustering tasks (\%). We \textbf{bold} the top one and \underline{underline} the runner-up in the average row.} \label{tab:nlp_tasks_cluster}
\end{table}

\begin{table}[t]
\footnotesize
\setlength{\tabcolsep}{4pt}
\centering
\scalebox{1.0}{
\begin{tabular}{@{}lcccc@{}}
\toprule
\textbf{Task} & \textbf{Echo} & \textbf{Prompt-} & \textbf{TP w.} & \textbf{HTP} \\
& \textbf{Mean} & \textbf{EOL} & \textbf{PromptEOL} & \textbf{(Ours)} \\
\midrule
STS16 & 76.19 & 70.08 & 70.13 & 57.84 \\
STS15 & 69.09 & 69.25 & 69.93 & 62.58 \\
STS14 & 58.27 & 62.39 & 59.40 & 48.28 \\
STS13 & 71.91 & 74.19 & 75.80 & 54.68 \\
STS12 & 46.85 & 63.56 & 65.22 & 44.80 \\
\midrule
\textbf{Average} & \underline{64.46} & 67.89 & \textbf{68.10} & 53.64 \\
\bottomrule
\end{tabular}
}
\caption{Performance on STS tasks (\%). We \textbf{bold} the top one and \underline{underline} the runner-up in the average row.} \label{tab:nlp_tasks_sts}
\end{table}

\subsection{Ablations}
\label{app:ablation}

\cref{tab:performance_comparison} shows the detailed NDCG@10 for ablation study of various $K$ sizes. Note that when we only insert one token at the end of the last sentence, HTP is equivalant to TP Mean. 

\begin{table}[t]
    \centering
    \scalebox{0.7}{
    \begin{tabular}{cccccc}
        \toprule
        \textbf{Methods} & $K$ & \textbf{NFCorpus} & \textbf{SciFact} & \textbf{SummFD} & \textbf{2WikiMQA} \\
        \midrule
        TP Mean& - & 8.15 & 42.71 & 60.11 & 23.22 \\
        \midrule
        \multirow{7}{*}{HTP} & 1 & 15.51 & 46.66 & 54.21 & 28.44 \\
        & 2 & 16.25 &	45.00 &	56.49 &	28.91 \\
        & 4 & 13.16 &	45.11& 56.25 &	28.79 \\
        & 8 & 12.24 &	44.93 &	56.73 & 28.71 \\
        & 12 & 12.31 &	44.88 &	56.57 & 29.55 \\
        & 20 & 12.27 &	44.85 &	56.61 &	29.31 \\
        & 32 & 12.27 &	44.85 &	56.23 &	29.24 \\
        & 64 & 11.80 &	41.36 &	56.74 &	28.90 \\
        \bottomrule
    \end{tabular}
    }
        \caption{Performance comparison of the TP Mean and HTP methods across different datasets and varying $K$ values.} \label{tab:performance_comparison}
\end{table}

\section{Examples}

\label{app:examples}
\begin{tcolorbox}[
    colback=blue!5!white, 
    colframe=blue!75!black, 
    title=\textbf{Example on ArguAna}, 
    fonttitle=\bfseries,
    fontupper=\small 
]
\texttt{<s>} Given a claim, retrieve documents that support or refute the claim Text: \texttt{<B-PST>}\texttt{<B-PST>}\texttt{<B-PST>}\texttt{<B-PST>}\texttt{<B-PST>}\texttt{<B-PST>} \newline \texttt{<B-PST>}\texttt{<B-PST>}\texttt{<B-PST>}\texttt{<PST>} Ending poverty through entrepreneurialism Introducing finance provides communities with access to startup capital. \texttt{<PST>} Access to financial capital is vital in several respects for initiating capitalism. \texttt{<PST>} Firstly, access to capital enables entrepreneurialism. \texttt{<PST>} The poor have business ideas that would benefit both themselves and their community they just require access to capital to invest in such ideas. \texttt{<PST>} The Initiative ‘Lend with Care’ is providing access to capital to empower entrepreneurs. \texttt{<PST>} [1] . \texttt{<PST>} Secondly, the cumulative effect of small-scale savings and borrowing, enabled through microfinance enables individuals, families and communities, to enter markets - of land and property. \texttt{<PST>} Being able to buy property and land can enable personal security, dignity, and increasing returns. \texttt{<PST>} [1] See further readings: Lend with Care, 2013. \texttt{</s>}

\end{tcolorbox}

\begin{table}[t]
\centering
    \scalebox{0.56}{
\begin{tabular}{@{}llcccc@{}}
\toprule
\textbf{Dataset} & \textbf{Domain} & \textbf{\# Queries} & \textbf{\# Docs} & \textbf{\makecell{Avg. Query \\ Words}} & \textbf{\makecell{Avg. Doc \\ Words}} \\
\midrule
NarrativeQA      & Literature, Film & 10,449 & 355 & 9   & 50,474 \\
QMSum           & Meeting          & 1,527  & 197 & 71  & 10,058 \\
2WikiMultihopQA  & Wikipedia        & 300    & 300 & 12  & 6,132 \\
SummScreenFD     & ScreenWriting    & 336    & 336 & 102 & 5,582 \\
\bottomrule
\end{tabular}
}
\caption{Datasets statistics for LongEmbed.} \label{tab:longembed-stats}

\end{table}

\end{document}